\documentclass{article}
\usepackage{macros}
\usepackage[preprint]{tmlr/tmlrx}

\hypersetup{pdftitle={A Spectral Condition for Feature Learning}}

\usepackage[utf8]{inputenc}
\usepackage[T1]{fontenc}
\usepackage{hyperref}
\usepackage{url}
\usepackage{booktabs}
\usepackage{amsfonts}
\usepackage{nicefrac}
\usepackage{microtype}
\usepackage{xcolor}
\usepackage{thmtools}
\usepackage{thm-restate}
\definecolor{blue}{rgb}{0, 0, .7}
\definecolor{darkred}{rgb}{.5, 0, 0}
\definecolor{darkgreen}{rgb}{0, .5, 0}
\usepackage{multirow}

\crefrangeformat{assumption}{Assumptions~#3#1#4 to~#5#2#6}

\title{A Spectral Condition for Feature Learning}

\author{%
  Greg Yang\thanks{Equal contribution.} \\
  xAI\\
  \And
  James B.~Simon$^*$\\
  UC Berkeley \textit{\&} Imbue
  \And 
  Jeremy Bernstein$^*$ \\
  MIT
}

\begin{document}

\maketitle

\begin{abstract}

The push to train ever larger neural networks has motivated the study of initialization and training at large network width.
A key challenge is to scale training so that a network's internal representations evolve nontrivially at all widths, a process known as \emph{feature learning}.
Here, we show that feature learning is achieved by scaling the \textit{spectral norm} of weight matrices and their updates like $\smash{\sqrt{\texttt{fan-out}/\texttt{fan-in}}}$, in contrast to widely used but heuristic scalings based on Frobenius norm and entry size.
Our spectral scaling analysis also leads to an elementary derivation of \emph{maximal update parametrization}.
We develop this spectral perspective and structure the text with the intent of providing the reader with a solid conceptual understanding of the scaling behavior of feature learning in neural networks.


\end{abstract}

\section{Introduction}
\label{sec:intro}

Recent years have seen an unprecedented push to train deep learning systems with more and more parameters, leading to powerful models across domains and the unlocking of qualitatively new capabilities \citep{silver:2016, brown-2020-gpt3, ramesh:2022-dalle2}.
This continuing trend, combined with the technical challenges of training large models, has motivated much recent study of the dynamics of neural networks at \textit{large width}, and more generally the study of how their dynamics scale as network width grows.
This program has yielded a cornucopia of theoretical insights \citep{lee:2018-nngp, jacot:2018, arora:2019-cntk, canatar:2021-spectral-bias} and practical scaling recommendations \citep{yang:2021-tensor-programs-IV, yang2021tuning, dey:2023-cerebras-gpt}.

A key challenge when training a network of large width is to ensure that \textit{feature learning} occurs at hidden layers.
By this, we mean that the hyperparameters of the network are scaled in a manner such that the hidden representations of the network (as obtained by partial evaluation of the network up to a certain layer) change substantially over the course of training. Na\"ive hyperparameter scaling rules, including the well-studied ``neural tangent parametrization'' (NTP), in fact \textit{lose} feature learning at large width \citep{lee:2019-ntk, sohl-dickstein:2020-parameterization}. But  ample evidence supports the conclusion that proper feature learning is necessary for achieving optimal performance on many tasks \citep{lee:2020, fort:2020-deep-learning-versus-kernel-learning, atanasov:2021-silent-alignment, vyas:2022-limitations-of-the-ntk-for-generalization}. Furthermore, scaling training correctly can lead to new functionality such as \textit{hyperparameter transfer}. For instance, the recently proposed \textit{maximal update parametrization} \citep{yang:2021-tensor-programs-IV,yang2021tuning} allows for transferring hyperparameters from narrow models to wide models, avoiding the cost of tuning the wide model directly.

Maximal update parametrization ($\mu$P) is derived by fairly involved ``tensor programs'' arguments that track feature distributions analytically in the infinite width limit. Anecdotally, the principles underlying $\mu$P are not well understood by the community. In this paper, we provide a new perspective on $\mu$P, showing that its scaling relations can be obtained by elementary linear algebra arguments. In short, we show that $\mu$P is equivalent to scaling the \textit{spectral norm} of any weight matrix or update like $\smash{\sqrt{\texttt{fan-out}/\texttt{fan-in}}}$. This simple condition has various favorable numerical properties that contrast sharply with heuristic optimization strategies based on controlling the Frobenius norm \citep{lars} or entry size \citep{kingma_adam:_2015} of updates. In the authors' experience, the spectral scaling condition both simplifies the implementation of $\mu$P in code, and is significantly easier to work with theoretically, leading to further conceptual advances in our research \citep{agd-2023}.

On a more fundamental level, an important step to solving many problems in classical computer science is to write down a suitable distance function for the problem at hand \citep{bregman}. This idea is of particular importance in the design of optimization algorithms, where a notion of parameter distance is needed \citep{nemirovsky_yudin_1983, amari}. While it can be tempting to use the Euclidean norm on parameter vectors to measure distance, this na\"ive choice risks discarding the structure of the problem. For example, neural networks involve compositions of linear operators, which we refer to as their \textit{operator structure}. Past efforts to metrize the space of neural networks while accounting for their operator structure have included using the Frobenius norm to measure distance between matrices \citep{my-fromage}, which motivates various optimization algorithms that make Frobenius-normalized updates \citep{You2020Large, lars, pmlr-v80-shazeer18a, my-nero}. This paper shows that the spectral norm provides a better notion of distance between operators in the context of deep learning.

\subsection{Road map for the paper}

To begin, we state the precise conditions on hidden features that we wish to ensure.
We will ask for two things: both the \textit{features} and their \textit{updates} upon a step of gradient descent must be the proper size.
\begin{tcolorbox}[colback=white, colframe=black, boxrule=1pt, arc=0pt]
  \begin{desired}[Feature learning] \label{des:scaling} Let $\vh_\ell(\vx) \in\R^{\width_\ell}$ denote the features of input $\vx$ at layer $\ell$ of a neural network, and let $\Delta \vh_\ell(\vx)\in\R^{\width_\ell}$ denote their change after a gradient step. We desire that:
    \begin{align*}
        \norm{\vh_\ell}_2 = \Theta (\sqrt{\width_\ell})
        \quad \text{and} \quad
        \norm{\Delta \vh_\ell}_2 = \Theta (\sqrt{\width_\ell}),
     \quad \text{ at layers } \ell=1,...,L\!-\!1.
    \end{align*}
\end{desired}
\end{tcolorbox}
Let us unpack \cref{des:scaling}.
These conditions treat the \textit{$\ell^2$-norms} of the feature vectors $\vh_\ell(\vx)$ and $\Delta \vh_\ell(\vx)$, a framing which will prove convenient.
\cref{des:scaling} amounts to asking that the ``typical element size'' of vectors $\vh_\ell(\vx)$ and $\Delta \vh_\ell(\vx)$ is $\Theta(1)$ with respect to width $\width_\ell$ (we give a review of big-$\Theta$ notation in \cref{sec:preliminaries}). Enforcing that hidden features have $\Theta(1)$ element size has long been a principle of deep learning parametrization, motivated by the fact that activation functions are designed to take order-one inputs and give order-one outputs \citep{lecun:2002-backprop-tricks}.
Our second requirement stipulates that feature entries also undergo $\Theta(1)$ updates during training.
Note that any larger updates would blow up at large width, and any smaller updates would vanish at large width. We take \cref{des:scaling} as our definition of feature learning.\footnote{Our notion of feature learning might also be called ``nontrivial feature evolution.'' While other authors may prefer different notions of ``feature learning''---for example, the learning of interpretable, visualizable functions at hidden nodes \citep{zeiler:2014-visualizing-cnns, olah:2017}---nontrivial feature evolution in our sense is necessary for any other reasonable definition of the term.}

Our main message is that feature learning in the sense of \cref{des:scaling} may be ensured by the following \textit{spectral scaling condition} on the weight matrices of a deep network and their gradient updates:
\begin{tcolorbox}[colback=white, colframe=black, boxrule=1pt, arc=0pt]
  \begin{condition}[Spectral scaling] \label{cond:scaling} 
    Consider applying a gradient update $\DW_\ell \in \R^{\width_\ell\times \width_{\ell-1}}$ to the $\ell$th weight matrix $\mW_\ell\in\R^{\width_\ell \times \width_{\ell-1}}$.
    The spectral norms of these matrices should satisfy:
    \begin{align*}
        \norm{\mW_\ell}_* = \Theta \left( \sqrt{\frac{\width_\ell}{\width_{\ell-1}}} \right)
        \ \ \ \text{and} \ \ \
        \norm{\DW_\ell}_* = \Theta \left( \sqrt{\frac{\width_\ell}{\width_{\ell-1}}} \right),
     \quad \text{ at layers } \ell=1,...,L.
    \end{align*}
\end{condition}
\end{tcolorbox}
We review the spectral norm in \cref{sec:preliminaries}.
The spectral scaling condition has two components which will serve to enforce the respective components of \cref{des:scaling}.
The first component mandates that each \emph{weight matrix} has a spectral norm of a certain size, which will serve to enforce that the layer passes forward features of the correct size.
The second component mandates that each \emph{gradient update} has a spectral norm of a certain size, which will ensure that subsequent features undergo a change of the correct size.
We have implicitly assumed that the input has size $\norm{\vx}_2 = \Theta(\sqrt{\width_0})$, which is standard for image data. Language models are an important counterexample, where embedding matrices take one-hot inputs and the $\sqrt{\width_0}$ in \cref{cond:scaling} should be replaced by $1$.
\cref{app:nondim} provides a unifying treatment of these cases.

To get some quick intuition for the origin of \cref{cond:scaling}, observe that under the forward propagation
$\vh_\ell(\vx) = \mW_\ell \vh_{\ell-1}(\vx)$, if the layer input $\vh_{\ell-1}(\vx)$ aligns with the top singular vector of the weight matrix $\mW_\ell$, then
$\norm{\vh_\ell(\vx)}_2 = \norm{\mW_\ell}_* \cdot \norm{\vh_{\ell-1}(\vx)}_2$.
The requirement that
$\norm{\mW_\ell}_* = \Theta ( \sqrt{\width_\ell / \width_{\ell-1}} )$ then follows from \cref{des:scaling}.
The scaling of $\norm{\DW_\ell}_*$ can similarly be obtained by writing $\Delta \vh_\ell(\vx) = \DW_\ell \vh_{\ell-1}(\vx) + \ldots$ and applying the same argument. The key missing step is to justify that layer inputs actually do line up with the top singular subspaces of both weight matrices and weight updates. As the paper will show, gradient descent training actually induces this form of alignment.

The bulk of this paper is dedicated to thoroughly demonstrating that training in accordance with our spectral scaling condition satisfies \cref{des:scaling} in MLPs.
As an accessible path to this conclusion, we begin in \cref{sec:warmup_and_extensions} with a simple model---a deep linear MLP trained for one step on one example---and then successively extend to multiple training steps, a nonlinear model, and multiple inputs.
In the process, we give a scaling analysis of the dynamics of feature learning.
We then explain how \cref{cond:scaling} may be achieved in a standard deep learning setting and compare-and-contrast the resulting scaling prescription with others in the literature. In particular, we recover the recent ``maximal-update parametrization'' (\muP{}) \citep{yang:2021-tensor-programs-IV}.

\subsection{Summary of contributions}
Concretely, our contributions are as follows:
\begin{itemize}
    \item We propose the \textit{spectral scaling condition} (\cref{cond:scaling}) and show that it suffices to achieve feature learning in neural networks even at large width.
    \item We show how \cref{cond:scaling} may be implemented: either via direct spectral normalization, or by layerwise initialization scales $\{\sigma_\ell\}$ and learning rates $\{\eta_\ell\}$ that recover \textit{maximal update parametrization}.%
    \footnote{In fact, \cref{cond:scaling} is the \emph{unique} scaling on spectral norm that is equivalent to $\mu$P in its usual definition in terms of learning rate and initialization scaling (\cref{thm:recovermuP}).}
    \item We show that other popular scaling rules, including so-called \textit{standard parameterization} and \textit{neural tangent parametrization}, fail to satisfy \cref{cond:scaling}.
\end{itemize}


In the main text, we focus on MLPs trained via ordinary gradient descent for clarity. Our results may actually be extended to cover any architecture and any adaptive optimizer (for a suitable definition of \emph{any}, c.f.\  \cref{{sec:formal_theory}}).
Therefore our spectral scaling condition provides a unifying hyperparameter scaling rule that remains the same whether the underlying optimizer is, say, SGD or Adam. We suggest that, when one wishes to determine how the hyperparameters of a new deep learning system should scale with width, one might turn to the satisfaction of \cref{cond:scaling} as an overarching principle. 

\section{Preliminaries}
\label{sec:preliminaries}

Here we review standard notations which we use in our scaling analysis.

\textbf{Scaling notation.}
We will use the usual big-$O$ notation and variants to make statements about how various quantities scale with network width. Intuitively speaking:
\begin{itemize}
    \item $f(\width) = O(g(\width))$ means that $f(\width)$ ``scales no faster than'' $g(\width)$,
    \item $f(\width) = \Theta(g(\width))$ means that $f(\width)$ ``scales like'' or ``is order'' $g(\width)$,
    \item $f(\width) = \Omega(g(\width))$ means that $f(\width)$ ``scales at least as fast as'' $g(\width)$.
\end{itemize}
Formally, $f(\width) = \Theta(g(\width))$ is equivalent to the statement that there exist constants $c, C > 0$ such that $c \cdot g(\width) \le f(\width) \le C \cdot g(\width)$ for all sufficiently large $d$.
The weaker statements $f(\width) = O(g(\width))$ and $f(\width) = \Omega(g(\width))$ entail only the upper and lower bounds, respectively.

We will \textit{only} be concerned with scaling with respect to layer widths in this paper.
Big-$O$ notation will hide any dependence on other factors --- such as depth, dataset size, learning rate schedule, a global learning rate prefactor --- and our statements purely concern how quantities will or should scale with model width.

\textbf{Vector and matrix norms.}
We will use the standard $\ell^2$-norm $\norm{\cdot}_2$ to assess the size of a vector.
For matrices, we will principally use the \textit{spectral norm} $\norm{\cdot}_*$ (a.k.a.~\textit{operator norm}) defined as follows:
\begin{definition}[Spectral norm] \label{def:spectral_norm}
The spectral norm of a matrix $\mA \in \R^{m \times n}$ is given by
\begin{equation}
    \norm{\mA}_* \defeq
    \max_{\vv \in \R^{n} \backslash \{\mathbf{0}\}}
    \frac{\norm{\mA \vv}_2}{\norm{\vv}_2}.
\end{equation}
\end{definition}
That is, the spectral norm is the largest factor by which a matrix can increase the norm of a vector on which it acts.
The spectral norm of a matrix is equal to its largest singular value.
We will sometimes contrast the spectral norm with the \textit{Frobenius norm} $\norm{\cdot}_F$ given by
$\smash{\norm{\mA}_F^2 = \sum_{ij} A_{ij}^2}$.

\textbf{Properties of the spectral norm.}
Let $\mA, \mB \in \R^{m \times n}$ be arbitrary matrices and $\vv \in \R^n$ be an arbitrary vector. As with all norms, the spectral norm is \textit{subadditive}, meaning that it obeys the \textit{triangle inequality} $\norm{\mA + \mB}_* \leq \norm{\mA}_* + \norm{\mB}_*$.
The spectral norm is also \textit{submultiplicative} in the sense that $\norm{\mA \vv}_* \le \norm{\mA}_* \cdot \norm{\vv}_2$ and $\norm{\mA \mB}_* \le \norm{\mA}_* \cdot \norm{\mB}_*$.
If we interpret a vector $\vv \in \R^n$ as a $1\times n$ matrix, then the $\ell^2$, spectral and Frobenius norms are equivalent: $\norm{\vv}_2 = \norm{\vv}_* = \norm{\vv}_F$.

\textbf{Special cases of the spectral norm.}
For a \textit{rank-one} matrix $\mA$, which can be written as an outer-product $\mA = \vu \vv\T$, it holds that $\norm{\mA}_* = \norm{\mA}_F = \norm{\vu}_2 \cdot \norm{\vv}_2$. A matrix $\mB \in \R^{m \times n}$ is \textit{semi-orthogonal} if either $\mB\T \mB = \mI_n$ or $\mB \mB\T = \mI_m$.
A semi-orthogonal matrix has unit spectral norm: $\norm{\mB}_* = 1$.

\section{The spectral scaling condition induces feature learning}
\label{sec:warmup_and_extensions}

In this section, we show that our spectral scaling condition (\cref{cond:scaling}) achieves feature evolution of the correct scale in multilayer perceptrons (MLPs).
We begin with a toy example which conveys key intuitions and then give a series of extensions which recover a much more general case.

\subsection{Warmup: deep linear MLP, one step of SGD, on a single example}
\label{subsec:warmup}

We begin with a simple model: a deep linear MLP trained for one step on a single input.
While elementary, this example will be sufficient to capture the intuition for a much more general case.

\textbf{Model definition.}
We will study an MLP composed of $L$ successive linear transformations.
We consider a single input $\vx \in \R^{\width_0}$ with norm $\norm{\vx}_2 = \Theta(\sqrt{\width_0})$. The first hidden representation is $\vh_1(\vx) = \mW_1 \vx$, with the rest of the network following recursively as:
\begin{equation}
    \label{eqn:hk}
    \vh_{\ell}(\vx) = \mW_\ell \vh_{\ell-1}(\vx)
    \qquad
    \text{for } \ell = 2, \ldots, L.
\end{equation}
We let $\vh_L(\vx) \in \R^{\width_L}$ be the network output.
We will keep the input dimension $\width_0$ and output dimension $\width_L$ fixed and consider scaling with respect to the hidden dimensions $\width_1, \ldots, \width_{L \!-\! 1}$.

We let the global loss be $\L = g(\vh_L(\vx), \vy)$ where $g$ and $\vy$ are a loss function and target vector, respectively.
During training, we will take gradient steps at each layer as $\DW_\ell = - \eta_\ell \cdot \nabla_{\mW_\ell} \L$, where $\eta_\ell$ is a layerwise learning rate.
We will ultimately solve for the scale of $\eta_\ell$, but for now we will be content to discuss the perturbation $\DW_\ell$ directly.
By \cref{eqn:hk}, hidden vector updates at subsequent layers are related by:
\begin{equation}
    \label{eqn:delta_h}
    \vh_\ell(\vx) + \Delta \vh_\ell(\vx) = (\mW_\ell+\DW_\ell) (\vh_{\ell-1}(\vx) + \Delta \vh_{\ell-1}(\vx)).
\end{equation}

\textbf{Hidden vector sizes.}
To reiterate \cref{des:scaling}, we wish for features at the $\ell$th layer to have a norm which scales as $\norm{\vh_\ell(\vx)}_2 = \Theta(\sqrt{\width_\ell})$.
Upon a gradient update, this feature vector should undergo an update of size $\norm{\Delta \vh_\ell (\vx)}_2 = \Theta(\sqrt{\width_\ell})$.
We will show that \cref{cond:scaling} is sufficient to achieve these aims at all layers.

\textbf{Plan of attack.}
For simplicity, we will first focus on the \textit{first step of gradient descent} after random initialization.
We will argue recursively in depth, showing that if the features at layer $\ell-1$ and their updates satisfy \cref{des:scaling}, then so will those at layer $\ell$.
In order to verify the desired scalings, we will show upper and lower scaling bounds separately: we will first show that the features and their updates are not larger than asked by \cref{des:scaling}, and then show that they are in fact also not smaller than asked by \cref{des:scaling}.

\textbf{Hidden vector updates.}
By the subadditivity and submultiplicativity of the spectral norm, \cref{eqn:hk,eqn:delta_h} imply that:
\begin{align}
    \norm{\vh_{\ell}(\vx)}_2 &\le \norm{\mW_\ell}_* \cdot \norm{\vh_{\ell-1}(\vx)}_2 = \Theta(\sqrt{\width_\ell}); \\
    \norm{\Delta \vh_\ell(\vx)}_2 &\le
    \norm{\DW_\ell}_* \cdot \norm{\vh_{\ell-1}(\vx)}_2
    + \norm{\mW_\ell}_* \cdot \norm{\Delta \vh_{\ell-1}(\vx)}_2
    + \norm{\DW_\ell}_* \cdot \norm{\Delta \vh_{\ell-1}(\vx)}_2
    = \Theta(\sqrt{\width_\ell}), \label{eqn:dh_upper_bound}
\end{align}
where on the right hand sides of the inequality we have inserted \cref{des:scaling} and \cref{cond:scaling}.
The spectral scaling condition thus gives features and feature updates obeying the correct \textit{upper} bounds, and we need merely show comparable lower bounds.

\textbf{Tightness of bounds via matrix-vector alignment.}
The upper bound in the submultiplicativity property $\norm{\mA \vv}_2 \le \norm{\mA}_* \cdot \norm{\vv}_2$ can be very loose---in particular, this is the case when the vector only interacts with the small singular values in the matrix.
We will now show that this is not the case in deep network training, and that these upper bounds provide a fairly accurate description of the way things scale.
We make two observations regarding random weight matrices and gradient updates:
\begin{claim}[Alignment of initial weight matrices]\label{claim:init}
Fix a feature vector $\vh_{\ell-1}(\vx)\in\R^{\width_{\ell-1}}$.
Assume $\mW_\ell$ in $\smash{\R^{\width_\ell\times \width_{\ell-1}}}$ is sampled using a common weight initialization strategy (e.g., Gaussian or semi-orthogonal init). Provided that fan-out is no less than fan-in
($\width_\ell \ge \width_{\ell-1}$),
then with high probability:
    \begin{equation*}
        \norm{ \mW_{\ell}\vh_{\ell-1}(\vx)}_2 = \Theta (\norm{\mW_{\ell}}_* \cdot \norm{\vh_{\ell-1}(\vx)}_2).
    \end{equation*}
\end{claim}
\begin{claim}[Alignment of updates]
\label{claim:update}
For an update $\DW_\ell$ given by gradient descent with batch size 1,
    \begin{equation*}
        \norm{ \Delta\mW_{\ell}\vh_{\ell-1}(\vx)}_2 = \norm{\DW_{\ell}}_* \cdot \norm{\vh_{\ell-1}(\vx)}_2.
    \end{equation*}
\end{claim}
In words, \cref{claim:init} states that random hidden weight matrices scale incoming vectors by factors commensurate to their spectral norms, so long as their fan-out is not smaller than their fan-in, which we will assume is the case for all but the final layer of the network.
\cref{claim:update} states the same of weight updates, but the proportionality constant is precisely one and requires no condition on dimensionality.%
\footnote{Note that $\vx$ in \cref{claim:update} is the same input that induced the gradient $\Delta\mW_{\ell}$ in the previous step. \cref{claim:update} is generally not an equality if this not true.}

We now justify these claims in turn.
For \cref{claim:init}, first suppose that $\mW_\ell$ is a random semi-orthogonal matrix as is a popular initialization strategy. Then all singular values of $\mW_\ell$ are one and, since fan-out exceeds fan-in, the null-space of $\mW_\ell$ is empty. Taken together, these observations imply the equality: $\norm{ \mW_{\ell}\vh_{\ell-1}(\vx)}_2 = \norm{\mW_{\ell}}_* \cdot \norm{\vh_{\ell-1}(\vx)}_2$.
Fortunately, if the elements of $\mW_\ell$ are instead sampled i.i.d.\ from a centered Gaussian distribution with standard deviation $\sigma_\ell$, then the situation is similar.
It is easily shown by the law of large numbers that
$\norm{\mW_\ell \vh_{\ell-1}(\vx)}_2 \approx \sigma_\ell \sqrt{\width_\ell} \norm{\vh_{\ell-1}(\vx)}_2$,
and it is a standard result in random matrix theory that
$\norm{\mW_\ell}_* \approx \sigma_\ell (\sqrt{\width_{\ell-1}} + \sqrt{\width_\ell})$
\citep{rudelson, vershynin_2018}.
\cref{claim:init} for Gaussian intialization follows by combining these results.

\cref{claim:update} is easily verified as follows.
Observe that we can write the update at layer $\ell$ as the outer-product:
\begin{equation} \label{eqn:DW}
    \DW_\ell = -\eta_\ell \cdot \nabla_{\vh_{\ell}(\vx)}\el \cdot \vh_{\ell-1}(\vx)^\top.
\end{equation}
So the update $\DW_\ell$ is rank-one with right singular vector $\vh_{\ell-1}(\vx)$.
To verify the claim, observe that:
\begin{equation}
    \norm{\DW_\ell \vh_{\ell-1}(\vx)}_2 = \eta_\ell \cdot
    \norm{\nabla_{\vh_{\ell}(\vx)}\el}_2
    \cdot \norm{\vh_{\ell-1}(\vx)}_2^2
    = \norm{\DW_\ell}_* \cdot \norm{\vh_{\ell-1}(\vx)}_2.
\end{equation}

With the claims established, we can now get lower bounds on hidden vector size which serve to verify \cref{des:scaling}.
The features at initialization scale correctly as:
\begin{equation} \label{eqn:vh_l_tight_scaling}
    \norm{\vh_\ell(\vx)}_2
    = \Theta \left( \norm{\mW_\ell}_* \cdot \norm{\vh_{\ell-1}(\vx)}_2 \right)
    = \Theta(\sqrt{\width_\ell}),
\end{equation}
where we have first used \cref{claim:init} and then inserted \cref{cond:scaling}.
To bound the size of $\Delta \vh_\ell(\vx)$, let us first observe from \cref{eqn:delta_h} that
\begin{equation}
\Delta \vh_\ell(\vx) = \DW_\ell \vh_{\ell-1}(\vx) + \mW_\ell \Delta \vh_{\ell-1}(\vx) + \DW_\ell \Delta \vh_{\ell-1}(\vx).
\end{equation}
So long as the first term $\DW_\ell \vh_{\ell-1}(\vx)$ does not perfectly cancel with the latter two, we have that:
\begin{equation}\label{eq:lower}
    \norm{\Delta \vh_\ell(\vx)}_2
    = \Omega ( \norm{\DW_\ell}_* \cdot \norm{\vh_{\ell-1}}_2 )
    = \Omega( \sqrt{\width_\ell} ),
\end{equation}
where in the last step we have inserted \cref{cond:scaling}.
Combining \cref{eq:lower} with our matching upper bound from \cref{eqn:dh_upper_bound},
we conclude that $\norm{\Delta \vh_\ell(\vx)}_2 = \Theta(\sqrt{\width_\ell})$ as desired.

We have achieved both clauses of Desideratum \ref{des:scaling} at layer $\ell$, completing a recursive step from layer $\ell-1$.
The norm of the input $\vx$ is by assumption the correct size to serve as a base case, and thus we recursively have the correct feature scaling at all layers.

\subsubsection{Key intuitions}
We now pause to discuss key intuitions from the above argument which will carry through to the general case.

\textbf{Weight updates are low-rank and aligned.}
An important observation is that weight updates are highly structured: they have low rank and align to incoming vectors.
This motivates the spectral norm (which is the degree by which a matrix scales a ``perfectly aligned'' vector) as the correct measure of size.

\textbf{Spectral variables enable simpler scaling analysis.}
Most prior work on hyperparameter scaling \citep{yang:2021-tensor-programs-IV,yaida:2022-hyperparameter-scaling-families} discusses layerwise initialization scales $\{\sigma_\ell\}$ and learning rates $\{\eta_\ell\}$ as the primary variables---although there are exceptions 
\citep{my-fromage}.
By contrast, we work directly in terms of quantities that these hyperparameters regulate: the spectral norms of $\mW_\ell$ and $\DW_\ell$.
This enables us to determine the sizes of hidden vectors and their updates quite easily, whereas the same calculation in terms of $\sigma_\ell$ and $\eta_\ell$ is more involved.
\cref{sec:implementation} shows how to recover $\sigma_\ell$ and $\eta_\ell$ from our spectral scaling condition.

\subsection{Extensions: additional gradient steps, nonlinearities, and multiple examples}

We now extend our warmup example to successively more complex settings, ultimately recovering the general case.
As we add back complexity, our spectral scaling condition will remain sufficient to achieve feature evolution of the proper size, and key intuitions from our warmup will continue to hold up to minor modifications.
Each extension requires making a natural assumption. We empirically verify these assumptions for a deep MLP in \cref{app:checking_assumptions}.

\subsubsection{Additional gradient steps}

Our warmup argument relied on two properties of $\mW_\ell$ at initialization: its spectral norm being the correct size (\cref{cond:scaling}) and it passing forward features $\vh_\ell(\vx)$ of the correct size (\cref{claim:update}).
Fortunately, it is easily seen that these two properties remain true of $\mW_\ell$ after a gradient step, and we can therefore treat the second (and later) steps exactly as we did the first.
This follows quickly given the following assumption.
\begin{assumption} \label{assumption_1}
Updates do not perfectly cancel initial quantities.
That is:
\begin{align}
    \norm{\mW_\ell + \DW_\ell}_*
    &=
    \Theta \left( \norm{\mW_\ell}_* + \norm{\DW_\ell}_* \right) \\
    \norm{\vh_\ell(\vx) + \Delta \vh_\ell(\vx)}_2
    &=
    \Theta(\norm{\vh_\ell(\vx)}_2 + \norm{\Delta \vh_\ell(\vx)}_2).
\end{align}
\end{assumption}
The sort of perfect cancellation required to violate this assumption will be rare in practice (and adding a small amount of randomness to the learning rate $\eta_\ell$ will fix any occurrence with high probability).
It follows that
$\norm{\mW_\ell + \DW_\ell}_* = \Theta( \sqrt{\width_\ell / \width_{\ell-1}})$
and
$\norm{\vh_\ell(\vx) + \Delta \vh_\ell(\vx)}_2 = \Theta(\sqrt{\width_\ell}).$
With these facts in place, the same argument we used in \cref{subsec:warmup} for the first step ensures that \cref{des:scaling} also holds at later steps.

\subsubsection{Nonlinearities}

We now add a nonlinearity $\phi$ to each layer of our MLP.
The modified forward recursion relation is:
\begin{equation}
    \vh_\ell(\vx) = \mW_\ell \vh'_{\ell-1}(\vx),
    \qquad
    \vh'_\ell(\vx) = \phi(\vh_\ell(\vx));
    \qquad
    \text{for } \ell = 2, \ldots, L\!-\!1,
\end{equation}
with base case $\vh_1(\vx) = \mW_1 \vx$ and output $\vh_L(\vx) = \mW_L \vh'_{L-1}(\vx)$.
We assume that the hidden features before and after the application of the nonlinearity are of the same scale:
\begin{assumption} \label{assumption_2}
    $\norm{\vh'_\ell(\vx)}_2 = \Theta\left(\norm{\vh_\ell(\vx)}_2\right).$
\end{assumption}
This is the expected behavior for most activation functions (which are designed to take in order-one inputs and return outputs which neither explode nor uniformly vanish) and seems like a reasonable assumption.
As in the linear case, $\DW_\ell$ will be rank-one and align to incoming signal as:
\begin{equation} \label{eqn:nonlin_update}
    \norm{\DW_\ell \vh'_{\ell-1}(\vx)}_2 = \norm{\DW_\ell}_* \cdot \norm{\vh'_{\ell-1}(\vx)}_2
\end{equation}
at each step.
All our scaling arguments from the linear case therefore carry through:
the term $\DW_\ell \vh'_{\ell-1}(\vx)$ is sufficient to induce a change $\Delta \vh_\ell(\vx)$ satisfying \cref{des:scaling}.
(The other terms which depend on $\Delta \vh'_{\ell-1}(\vx)$ will be no larger.)
\cref{cond:scaling} therefore still achieves correctly-scaled feature evolution.

\subsubsection{Batch size greater than one}

When training on a minibatch $\mathcal{D} = \{(\vx_i, \vy_i)\}_{i=1}^B$ with size $B > 1$, each gradient step is simply an average of the steps on each example:
\begin{equation} \label{eqn:DWk_sum}
    \DW_\ell = \frac{1}{B} \sum_{i=1}^B \DW_\ell^{(i)},
\end{equation}
where the sum runs over the minibatch index and $\DW_\ell^{(i)}$ denotes the update that would result from a step on the single example $(\vx_i, \vy_i)$.
While $\DW_\ell$ is generally no longer rank one and cannot perfectly align to all $B$ incoming vectors, Equation \ref{eqn:DWk_sum} makes it clear that at least one summand will align to each $\vh_\ell(\vx_i)$ from the batch.
We first assume that this term is not perfectly cancelled by the others:
\begin{assumption} \label{assumption_3}
    $\norm{\DW_\ell \vh_\ell(\vx_i)}_2 =
    \Theta(
    |\!|
    \frac{1}{B}
    \DW_\ell^{(i)} \vh_\ell(\vx_i) |\!|_2
    )$.
\end{assumption}
We additionally make the assumption that the batch size is fixed and independent of width:
\begin{assumption} \label{assumption_4}
The batch size is width-independent: $B = \Theta(1)$.
\end{assumption}

Combining these assumptions, we find that
$\norm{\DW_\ell \vh_\ell(\vx_i)}_2
= \Theta \left(
\norm{\DW_\ell}_* \cdot
\norm{\vh_\ell(\vx_i)}_2
\right)$: the batch update is aligned with incoming signal in a scaling sense.
Our previous scaling arguments in fact needed only alignment in this ``big-$\Theta$'' sense (as opposed to the perfect rank-one sense of \cref{claim:update}), and thus they still carry through: our spectral scaling condition continues to suffice to achieve proper feature learning as per \cref{des:scaling}.

\textbf{Empirical observation: low-rank structure remains at large batch size.}
Surprisingly, we observe numerically that MLP updates remain low (effective) rank and aligned with incoming vectors \textit{even at large batch size} $B$.
This is demonstrated in Figure \ref{fig:batch_scaling}.


\begin{figure}
  \centering
  \includegraphics[width=\textwidth]{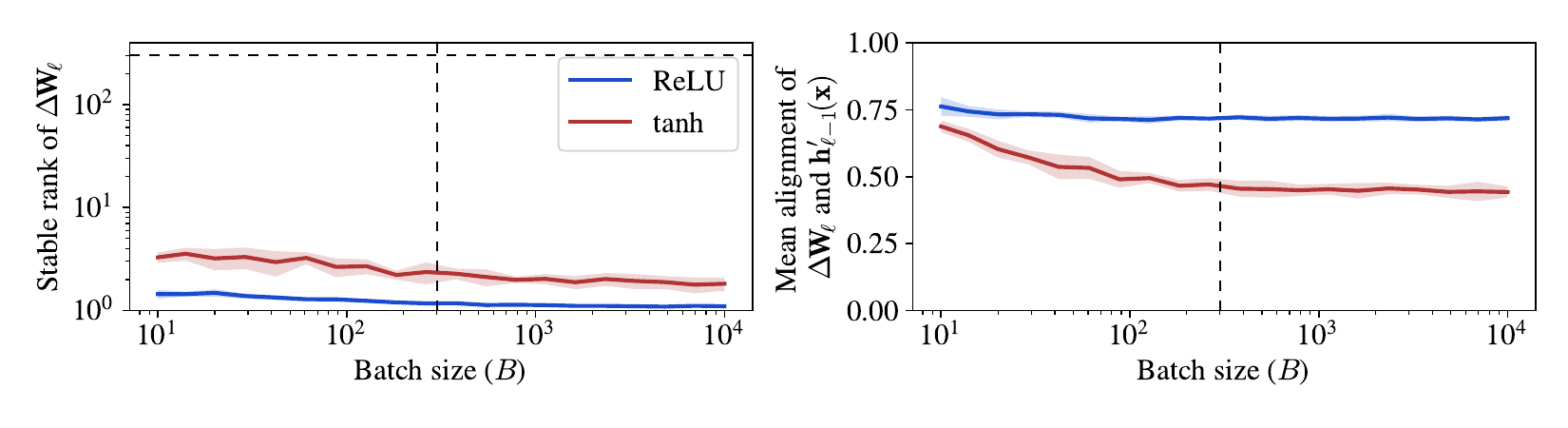}
  \caption{
    \textbf{Gradient updates have low effective rank and high-alignment with incoming hidden vectors even at large batch size in MLPs on CIFAR-10.}
    We randomly initialize MLPs with depth $L=3$, hidden widths $\width_1 = \width_2 = 300$, and $\textrm{ReLU}$ and $\textrm{tanh}$ activation functions.
    We then compute gradient updates $\DW_\ell$ for layer $\ell = 2$ on randomly-sampled size-$B$ subsets of CIFAR-10.
    \textbf{Left.}
    As a measure of effective rank, we report the \textit{stable rank}
    $\textrm{srank}(\DW_\ell) \defeq \norm{\DW_\ell}_F^2 / \norm{\DW_\ell}_*^2$
    of gradient updates.
    The stable rank remains less than 10 even when $B$ is large, which is much less than its maximal possible value of $\min(\width_1, \width_2) = 300$.
    \textbf{Right.}
    We report the average alignment $\norm{\DW_\ell \vh_{\ell-1}(\vx)}_2 / \norm{\DW_\ell}_* \norm{\vh'_{\ell-1}(\vx)}_2$ of the weight update $\DW_\ell$ to the incoming vector $\vh'_{\ell-1}(\vx)$, averaged over $\vx$ from the batch.
    Observe that alignment does not decay substantially with large batch size.
    Dashed lines on both axes in both subplots show the network width.
    Shaded regions denote one standard deviation of variation over random initializations and batches.
    Curves look similar after training.
    See \cref{app:exp} for experimental details.
  }
  \label{fig:batch_scaling}
\end{figure}

\subsection{Adam and other adaptive optimizers}

Adam (like most adaptive optimizers in deep learning) processes gradients into updates via an entrywise function.
For example, the momentum-less version of Adam is SignSGD \citep{bernstein_signsgd_2018} which just applies the sign function to each gradient entry.
Via less elementary arguments detailed in \cref{sec:formal_theory}, all the above discussion holds for $\Delta \mW_\ell$ calculated from these optimizers.
In a gist, the main ingredients are 1) the nontrivial insight from \emph{Tensor Programs} that gradients look like outer products of iid vectors, and 2) the fact that entrywise processing preserves the Frobenius norm of such a matrix up to multiplicative constants when width is large.

\subsection{Uniqueness of spectral scaling condition}
Technically, \cref{des:scaling} can be satisfied by scalings other than \cref{cond:scaling}.
For example, if one implements \cref{cond:scaling} at the first layer but sets $\norm{\DW_\ell}_* = 0$ for layers $\ell = 2,\ldots, L$, then \cref{des:scaling} still holds.
However, \cref{cond:scaling} is the unique \emph{maximal} scaling: if any of $\norm{\DW_\ell}_*$ or $\norm{\mW_\ell}_*$ exceeds \cref{cond:scaling}, then training will blow up as width is increased.
See \cref{sec:formal_theory} for a proof.

\renewcommand{\arraystretch}{1.5}
\begin{table}[ht]
    \centering
    \label{tab:norm_conversion}
    \begin{tabular}{ccccccc}
    \toprule
    \textbf{Matrix} & \textbf{Shape} & \textbf{Stable Rank} & \textbf{Spectral Norm} & \textbf{Frobenius Norm} & \textbf{RMS Entry Size} \\[0.25ex]
    \midrule
    $\mW_\ell$ & $\width \times \width$ & $\width$ & $1$ & $\sqrt{\width}$ & $\sqrt{1/\width}$\\
    $\DW_\ell$ & $\width \times \width$ & $1$ & $1$ & $1$ & $1/\width$\\
    \midrule
    $\mW_\ell$ & $1 \times \width$ & $1$ & $\sqrt{1/\width}$ & $\sqrt{1/\width}$ & $1/\width$\\
    $\DW_\ell$ & $1 \times \width$ & $1$ & $\sqrt{1/\width}$ & $\sqrt{1/\width}$ & $1/\width$\\
    \midrule
    $\mW_\ell$ & $\width_\ell \times \width_{\ell-1}$ & $\min(\width_\ell, \width_{\ell-1})$ & $\sqrt{\frac{\width_\ell}{\width_{\ell-1}}}$ & $\sqrt{\min(\width_\ell, \width_{\ell-1})\cdot\frac{\width_\ell}{\width_{\ell-1}}}$ & $\sqrt{\frac{\min(\width_\ell, \width_{\ell-1})}{\width_\ell \times \width_{\ell-1}} \cdot\frac{\width_\ell}{\width_{\ell-1}}}$\\
    $\DW_\ell$ & $\width_\ell \times \width_{\ell-1}$ & $1$ & $\sqrt{\frac{\width_\ell}{\width_{\ell-1}}}$ & $\sqrt{\frac{\width_\ell}{\width_{\ell-1}}}$ & $\sqrt{\frac{1}{\width_\ell \times \width_{\ell-1}}\cdot\frac{\width_\ell}{\width_{\ell-1}}}$\\
    \bottomrule
    \\
    \end{tabular}
    \caption{\textbf{Conversion between matrix norms for initial weights $\mW_\ell$ and updates $\DW_\ell$}. Results are first given for square matrices and vectors, then finally generalized to matrices of all shapes. Table entries denote the $\Theta$-scaling of each quantity. To obtain the Frobenius norm, one multiplies the spectral norm by the square root of the stable rank. To obtain the root-mean-square entry size, one divides the Frobenius norm by the square root of the number of entries. Notice that a $n\times \width$ initial matrix $\mW_\ell$ with $\Theta(1)$ spectral norm has $\Theta(\sqrt{1/\width})$ entry size, whereas a $n\times \width$ update $\DW_\ell$ with $\Theta(1)$ spectral norm has only $\Theta(1/\width)$ entry size, which is a factor $\sqrt{\width}$ smaller! The truth of the statement ``the weights move only a negligible amount in wide neural networks'' thus depends on whether this change is measured in entry size or spectral norm.}
\end{table}
\renewcommand{\arraystretch}{1}

\section{Efficient implementation of the spectral scaling condition}
\label{sec:implementation}

We have thus far proposed the \textit{spectral scaling condition} (\cref{cond:scaling}) and argued in \cref{sec:warmup_and_extensions} that this condition induces feature learning.
We now explain how this condition can be implemented in practice. Several implementation strategies are possible because relations exist between different notions of matrix norm (\cref{tab:norm_conversion} provides a summary). 
In this section, we first discuss an implementation via direct spectral normalization, and ultimately recover a more standard setup in which each layer's weight matrix is controlled by an initialization scale $\sigma_\ell$ and a (SGD) learning rate $\eta_\ell$ as follows:

\begin{tcolorbox}[colback=white, colframe=black, boxrule=1pt, arc=0pt]
\begin{parametrization}[Spectral parametrization]
\label{pzn:spectral_pzn}
We claim that the spectral scaling condition (\cref{cond:scaling}) is satisfied and feature learning is achieved (as per \cref{des:scaling}) if the initialization scale and learning rate of each layer $\ell$ are chosen according to:
\begin{equation*}
    \sigma_\ell =
    \Theta
    \left(
    \frac{1}{\sqrt{\width_{\ell-1}}}
    \min
    \left\{
    1,
    \sqrt{\frac{\width_\ell}{\width_{\ell-1}}}
    \right\}
    \right);
    \qquad
    \qquad
    \eta_\ell = \Theta \left( \frac{\width_{\ell}}{\width_{\ell-1}} \right).
\end{equation*}
\end{parametrization}
\end{tcolorbox}


\textbf{Na\"{i}ve method: direct spectral normalization.}
The most na\"{i}ve way to impose \cref{cond:scaling} is to directly normalize the relevant quantities by their spectral norm.
For instance, to initialize the weight matrix $\smash{\mW_\ell\in\R^{\width_\ell\times \width_{\ell-1}}}$ at the $\ell$th layer to have spectral norm $\smash{\sqrt{\width_\ell/\width_{\ell-1}}}$, one could sample a temporary matrix $\mW^\prime_\ell$ using any standard initializer and then re-normalize according to:
\begin{align}\label{eq:spectral-init}
    \mW_\ell = \sigma \sqrt{\frac{\width_\ell}{\width_{\ell-1}}} \times \frac{\mW^\prime_\ell}{\norm{\mW^\prime_\ell}_*},
\end{align}
where $\sigma = \Theta(1)$ is a width-independent prefactor.
Similarly, to ensure that the gradient step $\DW_\ell$ also has a spectral norm of the proper size, one might spectrally normalize the gradient according to:
\begin{equation} \label{eq:spectral-update}
    \DW_\ell = - \eta \sqrt{\frac{\width_\ell}{\width_{\ell-1}}} \times \frac{\nabla_{\mW_\ell}\mathcal{L}}{\norm{\nabla_{\mW_\ell}\mathcal{L}}_*},
\end{equation}
where $\eta = \Theta(1)$ is a width-independent prefactor.

It is quick to check that normalizing as in \cref{eq:spectral-init,eq:spectral-update} will satisfy \cref{cond:scaling}, but computing spectral norms is expensive and (as we will now show) can be avoided entirely by working out how $\norm{\mW_\ell'}_*$ and $\norm{\nabla_{\mW_\ell}\mathcal{L}}_*$ will scale and dividing by the appropriate factor.

\textbf{Random initialization.}
Let us suppose that, as is common practice, $\mW_\ell$ is initialized as $\mW_\ell = \sigma_\ell \cdot \mW'_\ell$, where all elements of $\mW_\ell'$ are initialized i.i.d. from a normal distribution with mean zero and unit variance.
The spectral norm of a matrix thus constructed is roughly $\norm{\mW_\ell}_* \approx \sigma_\ell \cdot (\sqrt{\width_\ell} + \sqrt{\width_{\ell-1}})$ \citep{rudelson,vershynin_2018}.
To get the desired scaling $\smash{\norm{\mW_\ell}_* = \Theta(\sqrt{\width_\ell / \width_{\ell-1}})}$, we need merely choose
$\sigma_\ell = \Theta ( \sqrt{\width_\ell / \width_{\ell-1}} \cdot (\sqrt{\width_\ell} + \sqrt{\width_{\ell-1}})^{-1} )$.
Simplifying within the $\Theta(\cdot)$, we arrive at $\sigma_\ell$ scaled as in the spectral parametrization (\cref{pzn:spectral_pzn}).
Initializing weights with a prefactor $\sigma_\ell$ scaling in this manner achieves the correct spectral norm of $\mW_\ell$.
We note that the constant factor suppressed by the $\Theta(\cdot)$ here will usually be small---for example, a prefactor of $\sqrt{2}$ agrees with typical practice for ReLU networks at most layers.
If $\mW_\ell'$ is instead a random semi-orthogonal matrix, then we can simply use a prefactor $\sigma_\ell = \Theta(\sqrt{\width_\ell / \width_{\ell-1}})$.

\textbf{Gradient updates.}
Here we give two methods for obtaining weight updates $\DW_\ell$ with the correct spectral norm.
The first method is to note that, for a matrix with low stable rank, the Frobenius norm scales like the spectral norm (and is cheap to compute), so we may simply approximate $\norm{\nabla_{\mW_\ell}\mathcal{L}}_* \approx \norm{\nabla_{\mW_\ell}\mathcal{L}}_F$ and use \cref{eq:spectral-update} directly.
This approach is useful if one wants to avoid worrying about width-scaling pre-factors entirely.
The second method---on which we will spend more time---is to make standard updates
\begin{equation} \label{eqn:gra\width_update}
    \DW_\ell = - \eta_\ell \nabla_{\mW_\ell}\mathcal{L}
\end{equation}
and find a scaling of $\eta_\ell$ such that $\smash{\norm{\DW_\ell}_* = \Theta(\sqrt{\width_\ell / \width_{\ell-1}})}$.

The main challenge lies in finding the scaling of the gradient $\norm{\nabla_{\mW_\ell}\mathcal{L}}_*$.
Note that we expect each gradient update $\DW_\ell$ to induce a change $\norm{\Delta \vh_L(\vx)}_2 = \Theta(\sqrt{\width_L})$ in the output which induces a change $\Delta \mathcal{L} = \Theta(1)$ for common loss functions $\mathcal{L}$. Taylor expanding the loss to first order, we also expect that
\begin{equation}
    \Delta \mathcal{L}
    = \Theta (
    \langle \DW_\ell,
    \nabla_{\mW_\ell}\mathcal{L} \rangle)
    = \Theta (\norm{\DW_\ell}_F \cdot \norm{\nabla_{\mW_\ell}\mathcal{L}}_F)
    = \Theta \left(
    \norm{\DW_\ell}_* \cdot \norm{\nabla_{\mW_\ell}\mathcal{L}}_*
    \right),
\end{equation}
where $\langle \cdot , \cdot \rangle$ denotes the trace inner product and we have used the facts that the two arguments of the inner product are (a) proportional to each other and (b) low-rank.
Inserting $\Delta \mathcal{L} = \Theta(1)$ and the spectral scaling condition $\norm{\DW_\ell}_* = \Theta(\sqrt{\width_\ell / \width_{\ell-1}})$, we can conclude that
\begin{equation}
    \norm{\nabla_{\mW_\ell}\mathcal{L}}_* = \Theta(\sqrt{\width_{\ell-1} / \width_{\ell}}).
\end{equation}

This result may also be reached via direct layerwise-recursive analysis of the size of the gradient.
Returning to \cref{eqn:gra\width_update}, we now see that we achieve a properly-scaled update $\norm{\DW_\ell}_* = \Theta(\sqrt{\width_\ell / \width_{\ell-1}})$ if we take
$\eta_\ell = \Theta \left( \width_{\ell} / \width_{\ell-1} \right)$ as prescribed by \cref{pzn:spectral_pzn}.

In summary: training with layerwise initialization $\sigma_\ell$ and learning rate $\eta_\ell$ scaled as in our \cref{pzn:spectral_pzn} will implement the spectral scaling condition
and give features and feature evolution of the correct size.

\section{Comparisons to existing parametrizations}
\label{sec:comparisons}



We have shown that training in accordance with our \textit{spectral scaling condition} (\cref{cond:scaling}) at every layer suffices to achieve correctly-scaled feature evolution, and we have given width scalings for layerwise hyperparameters $\sigma_\ell$ and $\eta_\ell$ that suffice to put this condition into action in a standard deep learning paradigm (\cref{pzn:spectral_pzn}).
Here we compare this ``spectral parametrization''  with popular parametrizations.
We find that our parametrization recovers the ``maximal update parametrization'' (\muP{}) at all layers and is different to other parametrizations.

\subsection{Comparison with ``maximal update parametrization''}

Maximal update parametrization (\muP{}) was recently proposed as a scaling rule that retains feature learning even at infinite width. \muP{} as given in Table 3 of \citet{yang2021tuning} may be recovered from 
our \cref{pzn:spectral_pzn} by setting $\width_0 = \width_L = 1$ and $\width_1=\width_2=...=\width_{L-1}$. Our \cref{pzn:spectral_pzn} actually streamlines and generalizes \muP{}: we provide a unifying treatment for any rectangular matrix, rather than treating input, hidden and output layers separately. In other words, from a spectral point of view, no layer is special. Our parametrization also includes scaling with respect to the input dimension $\width_0$ and output dimension $\width_L$ (as opposed to neglecting them as agnostic $\Theta(1)$ quantities) and treats hidden widths of unequal dimension ($\width_{\ell-1} \neq \width_\ell$).

\subsection{Contrast to ``standard parametrization''}

At present, the vast majority of deep learning systems use either ``Kaiming,'' ``Xavier,'' or ``LeCun'' initialization \citep{he:2015-delving-deep-into-rectifiers, glorot:2010-difficulty-of-training-deep-nets, lecun:2002-backprop-tricks} with layer-independent learning rates. Generically, we refer to this as ``standard parametrization'' (SP), where layerwise initialization and learning rates scale as:
\begin{equation}\label{eq:sp}
    \sigma_\ell = \Theta(1 / \sqrt{\width_{\ell-1}})
    \qquad
    \text{and}
    \qquad
    \eta_\ell = \Theta(1).
\end{equation}
Notice that SP initialization exceeds \cref{pzn:spectral_pzn} in any layer with fan-out smaller than fan-in. This includes the final layer in sufficiently wide networks. So, while \cref{pzn:spectral_pzn} implies that weight matrices have spectral norm $\Theta(\sqrt{\texttt{fan-out}/\texttt{fan-in}})$ at initialization, under SP the spectral norms of certain layers are initialized larger than this. This means that, under SP, network outputs can blow up if training aligns the layer inputs with the top singular subspaces (and in fact this alignment generally occurs).

\subsection{Contrast to ``neural tangent parametrization''}

The neural tangent parametrization (NTP) of \citet{jacot:2018} parameterizes a weight matrix as $\mW_\ell / \sqrt{\width_{\ell-1}}$ where the entries of $\mW_\ell$ are sampled iid standard normal, and applies gradient descent with a layer-independent step-size. Since dividing a weight matrix by $\sqrt{\width_{\ell-1}}$ also divides the gradient of that layer by the same factor, NTP is equivalent to training in our setup with a layer-wise standard deviation and learning rate:
\begin{equation}\label{eq:ntp}
    \sigma_\ell = \Theta(1 / \sqrt{\width_{\ell-1}})
    \qquad
    \text{and}
    \qquad
    \eta_\ell = \Theta(1 / \width_{\ell-1}).
\end{equation}
Comparing \cref{eq:sp,eq:ntp}, we see that NTP shares the same initialization scaling as SP but uses a smaller step-size. To see the deficiency of NTP, notice the output layer $\sigma_L$ is $\sqrt{\width_{L-1}}$ larger than \cref{pzn:spectral_pzn}, so that by the linearity of backpropagation, the gradient to any middle layer $\mW_\ell$ is also $\sqrt{\width_{L-1}}$ larger.
Then the $1 / \width_{\ell-1}$ learning rate in NTP induces a change $\Delta \mW_\ell$ that is $\sqrt{\width_{L-1}}/\width_{\ell -1}$ (smaller) compared to \cref{pzn:spectral_pzn}, which prescribes a learning rate of $\eta_\ell = 1$.
Because \cref{pzn:spectral_pzn} guarantees $\norm{\Delta \mW_\ell}_* = \Theta(1)$, NTP causes $\Delta \mW_\ell$ to vanish in spectral norm as hidden widths $\width_{\ell-1}, \width_{L} \to \infty$.

It bears noting that an MLP parameterized with the NTP can be made to undergo feature evolution by simply rescaling the network output (and appropriately scaling down the global learning rate) \citep{chizat:2019-lazy-training, bordelon:2022-feature-learning-dmft}.
This operation transforms the NTP into \muP{}.

\subsection{Contrast to ``Frobenius-normalized updates''}

\cref{cond:scaling} mandates that the spectral norm of the update at layer $\ell$ be proportional to the spectral norm of the weight matrix to which it is applied: $\norm{\DW_\ell}_* \propto \norm{\mW_\ell}_*$. This contrasts with a body of optimization work \citep{lars,pmlr-v80-shazeer18a, my-fromage, my-madam, my-nero} that has suggested instead that the \textit{Frobenius norm} of the update at layer $\ell$ be proportional to the \textit{Frobenius norm} of the weight matrix to which it is applied: $\norm{\DW_\ell}_F \propto \norm{\mW_\ell}_F$. For instance, \citet{my-fromage} analysed the operator structure of deep neural networks and wrote down perturbation bounds on network activations in terms of perturbations to the weight matrices at each layer, and used these perturbation bounds to motivate making updates that are small in Frobenius norm.
The flaw in that analysis is the assumption that weight matrices and their perturbations have identical conditioning structure \citep[Condition 2 of Theorem 1]{my-fromage}, when in reality weight matrices have high stable rank and updates have $\Theta(1)$ stable rank.
In light of this fact, the Frobenius-normalized proportionality rule should be modified to $\smash{\norm{\DW_\ell}_F \propto \norm{\mW_\ell}_F / \sqrt{\min(\width_\ell, \width_{\ell-1})}}$ in order to see proper feature evolution.

\section[Demonstration: muP vs. the NTP]{Demonstration: \boldmuP{} versus NTP}
\label{sec:mup_vs_ntp}

Here we discuss a simple, illustrative experiment in which we directly verify that \muP{} obeys our spectral scaling condition (\cref{cond:scaling}) and achieves leading-order feature evolution (\cref{des:scaling}) and the NTP does not.
We do so via direct measurement of spectral quantities in MLPs of varying width trained on the same task. \cref{app:exp} provides full experimental details, and results are plotted \cref{fig:mup_vs_ntp_exp}.

\textbf{Model and data.}
We train MLPs with $L = 3$ linear layers, no biases, and ReLU activations.
For demonstration purposes, we train on a small subset of $B = 200$ examples from a two-class subset of CIFAR-10.
The model has input dimension $\width_0 = 3072$, output dimension $\width_3 = 1$, and uniform hidden dimension $\width_1 = \width_2 = \width$, with $\width$ varied between training runs.
We initialize and train each MLP twice with hyperparameters obtained using \muP{} and NTP scalings, respectively. We train networks of widths $\width \in [16, 4096]$ to near-zero training loss. After training, we compute various spectral quantities which we now discuss.

\begin{figure}
  \centering
  \includegraphics[width=\textwidth]{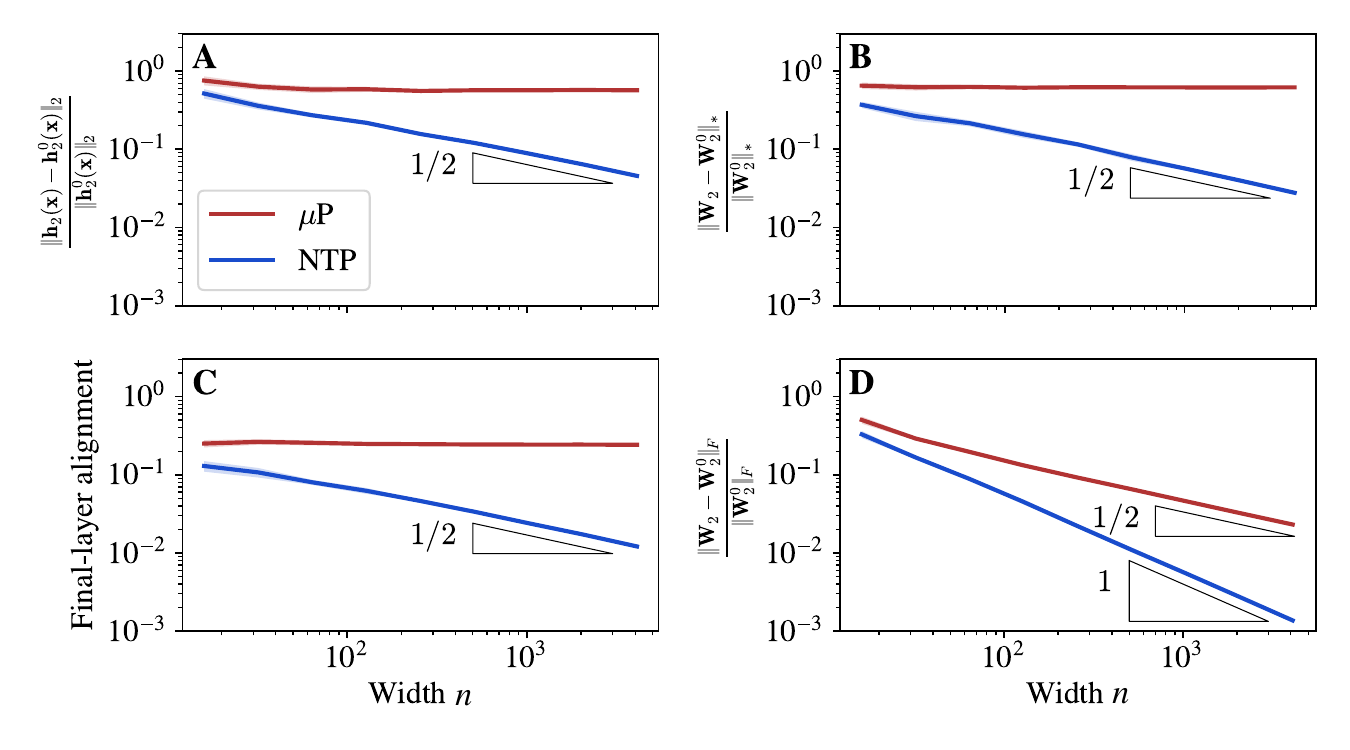}%
    \llap{\raisebox{5.2mm}{
    \includegraphics[width=2.5mm, clip=true, trim = 49mm 8mm 24mm 37mm]{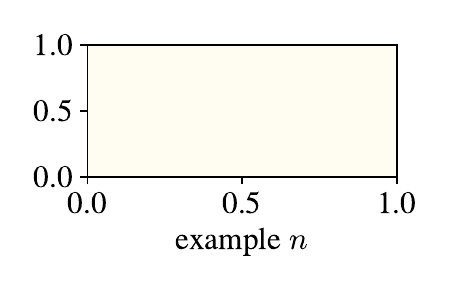}%
    }\hspace*{26.5mm}}%
    \llap{\raisebox{5.2mm}{
    \includegraphics[width=2.5mm, clip=true, trim = 49mm 8mm 24mm 37mm]{img/n_example.pdf}%
    }\hspace*{107mm}}%
  \caption{
  \textbf{Spectral quantities are $\bm{\Theta(1)}$ under \muP{} but decay with width under NTP.} We train multilayer perceptrons of varying width, and plot the following quantities computed between the initial and final network: 
  \textbf{(A)} Average relative change in features.
  \textbf{(B)} Relative change in weights in spectral norm.
  \textbf{(C)} Final layer alignment with incoming vectors (\cref{eqn:final_layer_align}).
  \textbf{(D)} Relative change in weights in Frobenius norm.
  Labeled triangles show predicted powerlaw slopes.
  Shaded regions show one standard deviation over random data and initialization.
  }
  \label{fig:mup_vs_ntp_exp}
\end{figure}

\textbf{Norms of feature updates.}
We measure the average relative change in features over training:
\begin{equation}
    \E{\vx}{
    \frac{\norm{\vh_2(\vx) - \vh_2^0(\vx)}_2}
    {\norm{\vh_2^0(\vx)}_2}
    },
\end{equation}
where $\vh_2^0(\vx)$ and $\vh_2(\vx)$ are the second (preactivation) hidden vector at initialization and after training respectively, and the expectation is over samples $\vx$ from the batch.
As shown in \cref{fig:mup_vs_ntp_exp}A, this feature evolution ratio remains roughly fixed as width $d$ grows when using \muP{}, in satisfaction of \cref{des:scaling}.
By contrast, it decays as $1 / \sqrt{\width}$ for the NTP as predicted by \citet{lee:2019-ntk}.

\textbf{Spectral norms of weight updates.}
We measure the relative change in weights in spectral norm:
$\norm{\mW_2 - \mW_2^0}_* / \norm{\mW_2^0}_*$,
where $\mW_2^0$ and $\mW_2$ are the second weight matrix before and after training, respectively.
As shown in \cref{fig:mup_vs_ntp_exp}B, this ratio remains roughly fixed with width in the case of \muP{}, in accordance with \cref{cond:scaling}.
By contrast, it decays as $1 / \sqrt{\width}$ for the NTP.

\textbf{Final-layer alignment.}
We measure the alignment of the final layer to incoming vectors as follows:
\begin{equation} \label{eqn:final_layer_align}
    \textrm{Final-layer alignment}
    \defeq
    \E{x}{
    \frac
    { \norm{\mW_3 \vh'_2(\vx)}_2 }
    { \norm{\mW_3}_* \cdot \norm{\vh'_2(\vx)}_2 }
    }.
\end{equation}
This quantity is $\Theta(1 / \sqrt{\width})$ at initialization.
As shown in \cref{fig:mup_vs_ntp_exp}C, it grows to $\Theta(1)$ when using \muP{} but remains $\Theta(1 / \sqrt{\width})$ when using the NTP.

\textbf{Frobenius norms of weight updates.}
One often hears the claim that ``the weights don't move'' when training very wide neural networks.
Here we show that the validity of this claim crucially depends on the choice of metric.
In \cref{fig:mup_vs_ntp_exp}D, we show that the Frobenius norm is deceptive: the net relative change $\norm{\mW_2 - \mW_2^0}_F / \norm{\mW_2^0}_F$ can decay with width even when the relative change in spectral norm is constant. This provides crucial context for interpreting existing results in the literature \citep[Figure 1]{lee:2019-ntk}.

\section{Related work}
\label{sec:related_work}

$\mu$P was derived heuristically from spectral norm considerations in talks given by the first author in 2021 \citep{yang2021workshoptalk}. Earlier work \citep{my-fromage} derived a spectral analysis of feature learning based on perturbation bounds, but that work obtained the wrong scaling relation with network width due to a flawed conditioning assumption on gradients. Below, we review various strands of related work on feature learning and training strategies.

\textbf{Parametrizations for wide neural networks.}
Much work has examined the scaling behavior of training dynamics of networks at large width.
Together, works on the ``neural tangent kernel'' (NTK) limit \citep{jacot:2018, lee:2019-ntk}, the ``mean-field'' limit \citep{rotskoff:2018-mean-field-limit, mei:2019-mean-field-limit, sirignano:2022-mean-field-analysis}, and the related ``feature learning'' limit \citep{geiger:2020-lazy-training, yaida:2022-hyperparameter-scaling-families, yang:2021-tensor-programs-IV} (i.e.\ the \muP{} limit), paint a rich picture of a family of possible infinite-width scalings.
After healthy debate regarding the relative empirical performance of the (more analytically tractable) NTK limit and the feature learning limit, a recent consensus holds that learning features is usually beneficial in practical large-scale deep learning settings \citep{chizat:2019-lazy-training, fort:2020-deep-learning-versus-kernel-learning, vyas:2022-limitations-of-the-ntk-for-generalization}.
Our spectral scaling analysis recovers the feature learning limit in a simpler manner than previous analyses.

\textbf{Spectral normalization.} Spectral normalization emerged as a form of weight normalization in the generative adversarial network literature \citep{Miyato2018SpectralNF}. This form of normalization acts on the weight matrices, and is used analogously to other normalization schemes such as batchnorm \citep{pmlr-v37-ioffe15} and layernorm \citep{Ba2016LayerN}.
In contrast to our spectral scaling condition (\cref{cond:scaling}), this method treats only the weights $\mW_\ell$, not the updates $\DW_\ell$. Furthermore, spectral normalization implementations typically set the spectral norm of weight matrices either to one or to a tunable hyperparameter \citep{farnia2018generalizable}, and do not include the key factor of $\sqrt{\width_\ell / \width_{\ell-1}}$ in our \cref{cond:scaling}.


\textbf{Operator theory of neural networks.} Neural networks are constructed by composing linear operators with elementwise nonlinearities. One line of work studies this operator structure and how it behaves under perturbation to understand how step sizes should be set in gradient descent. For instance, \citet{my-fromage} derive perturbation bounds on the maximum amount of feature change that can be induced by a gradient step in terms of the operator properties of the weight matrices. Meanwhile, \citet{yang:2021-tensor-programs-IV} study the operator structure of neural networks in the limit that width is taken to infinity, proposing a parametrization that obtains feature learning in this limit.

\paragraph{Optimization theory.}
A body of literature studies optimization algorithms for deep networks that take steps whose size is set relative to the weights to which they are applied \citep{lars, my-fromage, my-madam, my-nero, carbonnelle2019layer, pmlr-v80-shazeer18a}.
A particular focus has been placed on setting the Frobenius norm of update steps to be small relative to the Frobenius norm of the weight matrices \citep{lars,my-fromage,my-madam,my-nero}.
A main practical takeaway of this paper is that the Frobenius norm should be replaced by the spectral norm to get proper width scaling. The source of the difference between Frobenius and spectral norm is that gradient updates tend to have low stable rank, as shown in \cref{fig:batch_scaling}, while the weights themselves tend to have high stable rank.




\section{Conclusion}
\label{sec:discussion}

We have presented an analysis of the dynamics of feature learning in deep neural networks, beginning with desired conditions on feature evolution (\cref{des:scaling}) and culminating in the demonstration that these conditions may be achieved by simple scaling rules (\cref{cond:scaling,pzn:spectral_pzn}) applied uniformly to each layer.
Our analysis recovers and generalizes practically-important ``feature-learning parametrizations'' and provides a simple, unifying perspective on the question of parametrization in wide neural networks. For comparison, formal results derived under the \textit{tensor programs} framework are given in \cref{sec:formal_theory}.

Our discussion has focused principally on MLPs for clarity, but our feature learning desideratum and spectral scaling condition can be directly applied to structured architectures.
The spectral scaling condition may be applied to multi-index tensors as appear in convolutional architectures by applying the condition to appropriate ``slices'' of the full tensor.
Simple application of our spectral scaling condition recovers \muP{} scalings reported for these model classes (see e.g. \citet{yang2021tuning}, Table 8 and Section J.2).
We also give the hyperparameter scalings for biases (which are easily derived but omitted in the main text for clarity)
in Appendix \ref{app:biases}.
This architectural universality is also proven rigorously in \cref{sec:formal_theory}.

Finally, we note that under a natural redefinition of the $\ell^2$-norm of a vector $\vv\in\R^\width$ incorporating a normalization prefactor, \textit{all vector and matrix norms used in our spectral analysis} become $\Theta(1)$, permitting an elegant summary of our conclusions.
We state and discuss this nondimensionalization procedure in \cref{app:nondim}.
This generalization permits the extension of our results to e.g. the input layer in language modeling, in which one-hot embeddings violate our assumption that $\norm{\vx}_2 = \Theta(\sqrt{\width_0})$.

\section*{Acknowledgements}
The authors thank Josh Albrecht, Blake Bordelon, Alex Wei, Nikhil Ghosh, and Dhruva Karkada for useful discussions and comments on the manuscript.
JS gratefully acknowledges support from the National Science Foundation Graduate Fellow Research Program (NSF-GRFP) under grant DGE 1752814.

\section*{Author Contributions}
GY developed our core insight regarding the utility of the spectral norm, produced our tensor programs theory (\cref{sec:formal_theory}), and aided in refinement of the paper.
JS spearheaded the writing of the paper, led iteration towards simple analysis which communicates our spectral picture, and ran experiments.
JB developed an early incarnation of the spectral picture \citep{my-fromage}, contributed key insights simplifying our exposition including unifying all layers under single formulae, aided in writing the paper, and ran experiments.

\bibliography{refs, jamie_refs}

\begin{thebibliography}{51}
\providecommand{\natexlab}[1]{#1}
\providecommand{\url}[1]{\texttt{#1}}
\expandafter\ifx\csname urlstyle\endcsname\relax
  \providecommand{\doi}[1]{doi: #1}\else
  \providecommand{\doi}{doi: \begingroup \urlstyle{rm}\Url}\fi

\bibitem[Amari(1998)]{amari}
Shun-ichi Amari.
\newblock Natural gradient works efficiently in learning.
\newblock \emph{Neural Computation}, 1998.

\bibitem[Arora et~al.(2019)Arora, Du, Hu, Li, Salakhutdinov, and
  Wang]{arora:2019-cntk}
Sanjeev Arora, Simon~S. Du, Wei Hu, Zhiyuan Li, Ruslan Salakhutdinov, and
  Ruosong Wang.
\newblock On exact computation with an infinitely wide neural net.
\newblock In \emph{Neural Information Processing Systems}, 2019.

\bibitem[Atanasov et~al.(2022)Atanasov, Bordelon, and
  Pehlevan]{atanasov:2021-silent-alignment}
Alexander Atanasov, Blake Bordelon, and Cengiz Pehlevan.
\newblock Neural networks as kernel learners: the silent alignment effect.
\newblock In \emph{International Conference on Learning Representations}, 2022.

\bibitem[Ba et~al.(2016)Ba, Kiros, and Hinton]{Ba2016LayerN}
Jimmy Ba, Jamie~Ryan Kiros, and Geoffrey~E. Hinton.
\newblock Layer normalization.
\newblock \emph{arXiv:1607.06450}, 2016.

\bibitem[Bernstein et~al.(2018)Bernstein, Wang, Azizzadenesheli, and
  Anandkumar]{bernstein_signsgd_2018}
Jeremy Bernstein, Yu-Xiang Wang, Kamyar Azizzadenesheli, and Anima Anandkumar.
\newblock {signSGD}: {Compressed} {Optimisation} for {Non}-{Convex} {Problems},
  February 2018.
\newblock URL \url{https://arxiv.org/abs/1802.04434v3}.

\bibitem[Bernstein et~al.(2020{\natexlab{a}})Bernstein, Vahdat, Yue, and
  Liu]{my-fromage}
Jeremy Bernstein, Arash Vahdat, Yisong Yue, and Ming-Yu Liu.
\newblock On the distance between two neural networks and the stability of
  learning.
\newblock In \emph{Neural Information Processing Systems}, 2020{\natexlab{a}}.

\bibitem[Bernstein et~al.(2020{\natexlab{b}})Bernstein, Zhao, Meister, Liu,
  Anandkumar, and Yue]{my-madam}
Jeremy Bernstein, Jiawei Zhao, Markus Meister, Ming-Yu Liu, Anima Anandkumar,
  and Yisong Yue.
\newblock Learning compositional functions via multiplicative weight updates.
\newblock In \emph{Neural Information Processing Systems}, 2020{\natexlab{b}}.

\bibitem[Bernstein et~al.(2023)Bernstein, Mingard, Huang, Azizan, and
  Yue]{agd-2023}
Jeremy Bernstein, Chris Mingard, Kevin Huang, Navid Azizan, and Yisong Yue.
\newblock {A}utomatic {G}radient {D}escent: {D}eep {L}earning without
  {H}yperparameters.
\newblock \emph{arXiv:2304.05187}, 2023.

\bibitem[Bordelon \& Pehlevan(2022)Bordelon and
  Pehlevan]{bordelon:2022-feature-learning-dmft}
Blake Bordelon and Cengiz Pehlevan.
\newblock Self-consistent dynamical field theory of kernel evolution in wide
  neural networks.
\newblock In \emph{Neural Information Processing Systems}, 2022.

\bibitem[Brown et~al.(2020)Brown, Mann, Ryder, Subbiah, Kaplan, Dhariwal,
  Neelakantan, Shyam, Sastry, Askell, Agarwal, Herbert-Voss, Krueger, Henighan,
  Child, Ramesh, Ziegler, Wu, Winter, Hesse, Chen, Sigler, Litwin, Gray, Chess,
  Clark, Berner, McCandlish, Radford, Sutskever, and Amodei]{brown-2020-gpt3}
Tom Brown, Benjamin Mann, Nick Ryder, Melanie Subbiah, Jared~D Kaplan, Prafulla
  Dhariwal, Arvind Neelakantan, Pranav Shyam, Girish Sastry, Amanda Askell,
  Sandhini Agarwal, Ariel Herbert-Voss, Gretchen Krueger, Tom Henighan, Rewon
  Child, Aditya Ramesh, Daniel Ziegler, Jeffrey Wu, Clemens Winter, Chris
  Hesse, Mark Chen, Eric Sigler, Mateusz Litwin, Scott Gray, Benjamin Chess,
  Jack Clark, Christopher Berner, Sam McCandlish, Alec Radford, Ilya Sutskever,
  and Dario Amodei.
\newblock Language models are few-shot learners.
\newblock \emph{Neural Information Processing Systems}, 2020.

\bibitem[Canatar et~al.(2021)Canatar, Bordelon, and
  Pehlevan]{canatar:2021-spectral-bias}
Abdulkadir Canatar, Blake Bordelon, and Cengiz Pehlevan.
\newblock Spectral bias and task-model alignment explain generalization in
  kernel regression and infinitely wide neural networks.
\newblock \emph{Nature Communications}, 2021.

\bibitem[Carbonnelle \& Vleeschouwer(2019)Carbonnelle and
  Vleeschouwer]{carbonnelle2019layer}
Simon Carbonnelle and Christophe~De Vleeschouwer.
\newblock Layer rotation: {A} surprisingly simple indicator of generalization
  in deep networks?
\newblock In \emph{ICML Workshop on Identifying and Understanding Deep Learning
  Phenomena}, 2019.

\bibitem[Chizat et~al.(2019)Chizat, Oyallon, and
  Bach]{chizat:2019-lazy-training}
L\'{e}na\"{i}c Chizat, Edouard Oyallon, and Francis Bach.
\newblock On lazy training in differentiable programming.
\newblock \emph{Neural Information Processing Systems}, 2019.

\bibitem[Dey et~al.(2023)Dey, Gosal, Khachane, Marshall, Pathria, Tom,
  Hestness, et~al.]{dey:2023-cerebras-gpt}
Nolan Dey, Gurpreet Gosal, Hemant Khachane, William Marshall, Ribhu Pathria,
  Marvin Tom, Joel Hestness, et~al.
\newblock Cerebras-{GPT}: Open compute-optimal language models trained on the
  {C}erebras wafer-scale cluster.
\newblock \emph{arXiv:2304.03208}, 2023.

\bibitem[Dhillon \& Tropp(2008)Dhillon and Tropp]{bregman}
Inderjit~S. Dhillon and Joel~A. Tropp.
\newblock Matrix nearness problems with {Bregman} divergences.
\newblock \emph{SIAM Journal on Matrix Analysis and Applications}, 2008.

\bibitem[Farnia et~al.(2019)Farnia, Zhang, and Tse]{farnia2018generalizable}
Farzan Farnia, Jesse Zhang, and David Tse.
\newblock Generalizable adversarial training via spectral normalization.
\newblock In \emph{International Conference on Learning Representations}, 2019.

\bibitem[Fort et~al.(2020)Fort, Dziugaite, Paul, Kharaghani, Roy, and
  Ganguli]{fort:2020-deep-learning-versus-kernel-learning}
Stanislav Fort, Gintare~Karolina Dziugaite, Mansheej Paul, Sepideh Kharaghani,
  Daniel~M. Roy, and Surya Ganguli.
\newblock Deep learning versus kernel learning: an empirical study of loss
  landscape geometry and the time evolution of the neural tangent kernel.
\newblock \emph{Neural Information Processing Systems}, 2020.

\bibitem[Geiger et~al.(2020)Geiger, Spigler, Jacot, and
  Wyart]{geiger:2020-lazy-training}
Mario Geiger, Stefano Spigler, Arthur Jacot, and Matthieu Wyart.
\newblock Disentangling feature and lazy training in deep neural networks.
\newblock \emph{Journal of Statistical Mechanics: Theory and Experiment}, 2020.

\bibitem[Glorot \& Bengio(2010)Glorot and
  Bengio]{glorot:2010-difficulty-of-training-deep-nets}
Xavier Glorot and Yoshua Bengio.
\newblock Understanding the difficulty of training deep feedforward neural
  networks.
\newblock In \emph{International Conference on Artificial Intelligence and
  Statistics}, 2010.

\bibitem[He et~al.(2015)He, Zhang, Ren, and
  Sun]{he:2015-delving-deep-into-rectifiers}
Kaiming He, Xiangyu Zhang, Shaoqing Ren, and Jian Sun.
\newblock Delving deep into rectifiers: Surpassing human-level performance on
  {ImageNet} classification.
\newblock In \emph{International Conference on Computer Vision}, 2015.

\bibitem[Ioffe \& Szegedy(2015)Ioffe and Szegedy]{pmlr-v37-ioffe15}
Sergey Ioffe and Christian Szegedy.
\newblock Batch normalization: Accelerating deep network training by reducing
  internal covariate shift.
\newblock In \emph{International Conference on Machine Learning}, 2015.

\bibitem[Jacot et~al.(2018)Jacot, Hongler, and Gabriel]{jacot:2018}
Arthur Jacot, Cl{\'{e}}ment Hongler, and Franck Gabriel.
\newblock Neural tangent kernel: Convergence and generalization in neural
  networks.
\newblock In \emph{Neural Information Processing Systems}, 2018.

\bibitem[Kingma \& Ba(2015)Kingma and Ba]{kingma_adam:_2015}
Diederik~P. Kingma and Jimmy Ba.
\newblock Adam: {A} method for stochastic optimization.
\newblock In \emph{International Conference on Learning Representations}, 2015.

\bibitem[Krizhevsky(2009)]{krizhevsky:2009}
Alex Krizhevsky.
\newblock Learning multiple layers of features from tiny images.
\newblock Technical report, University of Toronto, 2009.

\bibitem[LeCun et~al.(2002)LeCun, Bottou, Orr, and
  M{\"u}ller]{lecun:2002-backprop-tricks}
Yann LeCun, L{\'e}on Bottou, Genevieve~B. Orr, and Klaus-Robert M{\"u}ller.
\newblock Efficient backprop.
\newblock In \emph{Neural Networks: Tricks of the Trade}. Springer, 2002.

\bibitem[Lee et~al.(2018)Lee, Bahri, Novak, Schoenholz, Pennington, and
  Sohl{-}Dickstein]{lee:2018-nngp}
Jaehoon Lee, Yasaman Bahri, Roman Novak, Samuel~S. Schoenholz, Jeffrey
  Pennington, and Jascha Sohl{-}Dickstein.
\newblock Deep neural networks as {G}aussian processes.
\newblock In \emph{International Conference on Learning Representations}, 2018.

\bibitem[Lee et~al.(2019)Lee, Xiao, Schoenholz, Bahri, Novak, Sohl{-}Dickstein,
  and Pennington]{lee:2019-ntk}
Jaehoon Lee, Lechao Xiao, Samuel~S. Schoenholz, Yasaman Bahri, Roman Novak,
  Jascha Sohl{-}Dickstein, and Jeffrey Pennington.
\newblock Wide neural networks of any depth evolve as linear models under
  gradient descent.
\newblock In \emph{Neural Information Processing Systems}, 2019.

\bibitem[Lee et~al.(2020)Lee, Schoenholz, Pennington, Adlam, Xiao, Novak, and
  Sohl{-}Dickstein]{lee:2020}
Jaehoon Lee, Samuel~S. Schoenholz, Jeffrey Pennington, Ben Adlam, Lechao Xiao,
  Roman Novak, and Jascha Sohl{-}Dickstein.
\newblock Finite versus infinite neural networks: an empirical study.
\newblock In \emph{Neural Information Processing Systems}, 2020.

\bibitem[Liu et~al.(2021)Liu, Bernstein, Meister, and Yue]{my-nero}
Yang Liu, Jeremy Bernstein, Markus Meister, and Yisong Yue.
\newblock Learning by turning: {N}eural architecture aware optimisation.
\newblock In \emph{International Conference on Machine Learning}, 2021.

\bibitem[Mei et~al.(2019)Mei, Misiakiewicz, and
  Montanari]{mei:2019-mean-field-limit}
Song Mei, Theodor Misiakiewicz, and Andrea Montanari.
\newblock Mean-field theory of two-layers neural networks: Dimension-free
  bounds and kernel limit.
\newblock In \emph{Conference on Learning Theory}, 2019.

\bibitem[Miyato et~al.(2018)Miyato, Kataoka, Koyama, and
  Yoshida]{Miyato2018SpectralNF}
Takeru Miyato, Toshiki Kataoka, Masanori Koyama, and Yuichi Yoshida.
\newblock Spectral normalization for {G}enerative {A}dversarial {N}etworks.
\newblock In \emph{International Conference on Learning Representations}, 2018.

\bibitem[Nemirovsky \& Yudin(1983)Nemirovsky and Yudin]{nemirovsky_yudin_1983}
Arkady~S. Nemirovsky and David~B. Yudin.
\newblock \emph{Problem complexity and method efficiency in optimization}.
\newblock Wiley, 1983.

\bibitem[Olah et~al.(2017)Olah, Mordvintsev, and Schubert]{olah:2017}
Chris Olah, Alexander Mordvintsev, and Ludwig Schubert.
\newblock Feature visualization.
\newblock \emph{Distill}, 2017.

\bibitem[Poole et~al.(2016)Poole, Lahiri, Raghu, Sohl-Dickstein, and
  Ganguli]{poole:2016}
Ben Poole, Subhaneil Lahiri, Maithra Raghu, Jascha Sohl-Dickstein, and Surya
  Ganguli.
\newblock Exponential expressivity in deep neural networks through transient
  chaos.
\newblock In \emph{Neural Information Processing Systems}, 2016.

\bibitem[Ramesh et~al.(2022)Ramesh, Dhariwal, Nichol, Chu, and
  Chen]{ramesh:2022-dalle2}
Aditya Ramesh, Prafulla Dhariwal, Alex Nichol, Casey Chu, and Mark Chen.
\newblock Hierarchical text-conditional image generation with {CLIP} latents.
\newblock \emph{arXiv:2204.06125}, 2022.

\bibitem[Rotskoff \& Vanden-Eijnden(2022)Rotskoff and
  Vanden-Eijnden]{rotskoff:2018-mean-field-limit}
Grant~M. Rotskoff and Eric Vanden-Eijnden.
\newblock Trainability and accuracy of artificial neural networks: An
  interacting particle system approach.
\newblock \emph{Communications on Pure and Applied Mathematics}, 2022.

\bibitem[Rudelson \& Vershynin(2010)Rudelson and Vershynin]{rudelson}
Mark Rudelson and Roman Vershynin.
\newblock Non-asymptotic theory of random matrices: {E}xtreme singular values.
\newblock In \emph{International Congress of Mathematicians}, 2010.

\bibitem[Shazeer \& Stern(2018)Shazeer and Stern]{pmlr-v80-shazeer18a}
Noam Shazeer and Mitchell Stern.
\newblock Adafactor: Adaptive learning rates with sublinear memory cost.
\newblock In \emph{International Conference on Machine Learning}, 2018.

\bibitem[Silver et~al.(2016)Silver, Huang, Maddison, Guez, Sifre, van~den
  Driessche, Schrittwieser, Antonoglou, Panneershelvam, Lanctot, Dieleman,
  Grewe, Nham, Kalchbrenner, Sutskever, Lillicrap, Leach, Kavukcuoglu, Graepel,
  and Hassabis]{silver:2016}
David Silver, Aja Huang, Chris~J. Maddison, Arthur Guez, Laurent Sifre, George
  van~den Driessche, Julian Schrittwieser, Ioannis Antonoglou, Veda
  Panneershelvam, Marc Lanctot, Sander Dieleman, Dominik Grewe, John Nham, Nal
  Kalchbrenner, Ilya Sutskever, Timothy Lillicrap, Madeleine Leach, Koray
  Kavukcuoglu, Thore Graepel, and Demis Hassabis.
\newblock Mastering the game of {Go} with deep neural networks and tree search.
\newblock \emph{Nature}, 2016.

\bibitem[Sirignano \& Spiliopoulos(2022)Sirignano and
  Spiliopoulos]{sirignano:2022-mean-field-analysis}
Justin Sirignano and Konstantinos Spiliopoulos.
\newblock Mean field analysis of deep neural networks.
\newblock \emph{Mathematics of Operations Research}, 2022.

\bibitem[Sohl-Dickstein et~al.(2020)Sohl-Dickstein, Novak, Schoenholz, and
  Lee]{sohl-dickstein:2020-parameterization}
Jascha Sohl-Dickstein, Roman Novak, Samuel~S. Schoenholz, and Jaehoon Lee.
\newblock On the infinite width limit of neural networks with a standard
  parameterization.
\newblock \emph{arXiv:2001.07301}, 2020.

\bibitem[Vershynin(2018)]{vershynin_2018}
Roman Vershynin.
\newblock \emph{High-Dimensional Probability: An Introduction with Applications
  in Data Science}.
\newblock Cambridge University Press, 2018.

\bibitem[Vyas et~al.(2022)Vyas, Bansal, and
  Nakkiran]{vyas:2022-limitations-of-the-ntk-for-generalization}
Nikhil Vyas, Yamini Bansal, and Preetum Nakkiran.
\newblock Limitations of the {NTK} for understanding generalization in deep
  learning.
\newblock \emph{arXiv:2206.10012}, 2022.

\bibitem[Yaida(2022)]{yaida:2022-hyperparameter-scaling-families}
Sho Yaida.
\newblock Meta-principled family of hyperparameter scaling strategies.
\newblock \emph{arXiv:2210.04909}, 2022.

\bibitem[Yang \& Hu(2021{\natexlab{a}})Yang and Hu]{yang2021workshoptalk}
Greg Yang and Edward~J. Hu.
\newblock Feature learning in infinite-width neural networks.
\newblock In \emph{ICML 2021 Workshop on Over-parameterization: Pitfalls and
  Opportunities}, 2021{\natexlab{a}}.
\newblock
  \url{https://slideslive.com/38963058/feature-learning-in-infinitewidth-neural-networks}.

\bibitem[Yang \& Hu(2021{\natexlab{b}})Yang and
  Hu]{yang:2021-tensor-programs-IV}
Greg Yang and Edward~J. Hu.
\newblock Tensor {P}rograms {IV}: Feature learning in infinite-width neural
  networks.
\newblock In \emph{International Conference on Machine Learning},
  2021{\natexlab{b}}.

\bibitem[Yang \& Littwin(2023)Yang and Littwin]{yang2023adaptive}
Greg Yang and Etai Littwin.
\newblock Tensor programs {IVb}: Adaptive optimization in the $\infty$-width
  limit.
\newblock \emph{arXiv:2308.01814}, 2023.

\bibitem[Yang et~al.(2021)Yang, Hu, Babuschkin, Sidor, Liu, Farhi, Ryder,
  Pachocki, Chen, and Gao]{yang2021tuning}
Greg Yang, Edward~J. Hu, Igor Babuschkin, Szymon Sidor, Xiaodong Liu, David
  Farhi, Nick Ryder, Jakub Pachocki, Weizhu Chen, and Jianfeng Gao.
\newblock Tensor {P}rograms {V}: Tuning large neural networks via zero-shot
  hyperparameter transfer.
\newblock In \emph{Neural Information Processing Systems}, 2021.

\bibitem[You et~al.(2017)You, Gitman, and Ginsburg]{lars}
Yang You, Igor Gitman, and Boris Ginsburg.
\newblock Scaling {SGD} batch size to 32{K} for {I}mage{N}et training.
\newblock Technical Report UCB/EECS-2017-156, University of California,
  Berkeley, 2017.

\bibitem[You et~al.(2020)You, Li, Reddi, Hseu, Kumar, Bhojanapalli, Song,
  Demmel, Keutzer, and Hsieh]{You2020Large}
Yang You, Jing Li, Sashank Reddi, Jonathan Hseu, Sanjiv Kumar, Srinadh
  Bhojanapalli, Xiaodan Song, James Demmel, Kurt Keutzer, and Cho-Jui Hsieh.
\newblock Large batch optimization for deep learning: Training {BERT} in 76
  minutes.
\newblock In \emph{International Conference on Learning Representations}, 2020.

\bibitem[Zeiler \& Fergus(2014)Zeiler and Fergus]{zeiler:2014-visualizing-cnns}
Matthew~D. Zeiler and Rob Fergus.
\newblock Visualizing and understanding convolutional networks.
\newblock In \emph{European Conference on Computer Vision}, 2014.

\end{thebibliography}
\bibliographystyle{tmlr/tmlr}

\appendix
\newpage

\section{Experimental details}
\label{app:exp}

\textbf{Experimental details for \cref{fig:batch_scaling}.}
We examine randomly-initialized MLPs with depth $L=3$, widths $\width_0 = 3072, \width_1 = \width_2 = d = 300, \width_3 = 10$, and ReLU and tanh activation functions.
We then pass single batches of CIFAR-10 data of varying size $B$ through the model and compute a single gradient step $\DW_\ell$ at layer $\ell = 2$ with arbitrary learning rate.
We then compute the stable rank of $\DW_\ell$ and the alignment metric
\begin{equation}
    \E{\vx}{
    \frac{\norm{\DW_\ell \vh_{\ell-1}(\vx_i)}_2}{\norm{\DW_\ell}_* \norm{\vh'_{\ell-1}(\vx_i)}_2}
    },
\end{equation}
where the expectation is taken over $\vx$ from the batch.
Shaded envelopes in \cref{fig:batch_scaling} denote one standard deviation with respect to both random network initialization and random batch selection over $10$ trials.

We conducted an additional experiment to try to ascertain the source of the low effective rank of $\DW_\ell$.
Let $\mH'_{\ell-1} = [\vh'_1(\vx_1) \ldots \vh'_1(\vx_B)] \in \R^{d \times B}$ be the matrix formed from stacking the $\ell=1$ post-nonlinearity hidden vectors from the full batch, and let $\mG_{\ell} = [\vg_2(\vx_1) \ldots \vg_2(\vx_B)] \in \R^{d \times B}$ with $\vg_\ell(\vx) = \nabla_{\vh_{\ell}(\vx)}\L$ be the matrix formed from stacking the loss gradients at layer $\ell=2$ from the full batch.
It is the case that $\DW_\ell \propto \mG_{\ell} {\mH'_{\ell-1}}\T$, and if $\DW_\ell$ has low stable rank, it is a reasonable guess that either $\mH'_{\ell-1}$ and $\mG_{\ell}$ also has low stable rank.
In fact, we find that \textit{both} $\mH'_{\ell-1}$ and $\mG_{\ell}$ have low stable rank, as shown in \cref{fig:forwar\width_backwar\width_sranks}.

\begin{figure}
  \centering
  \includegraphics[width=\textwidth]{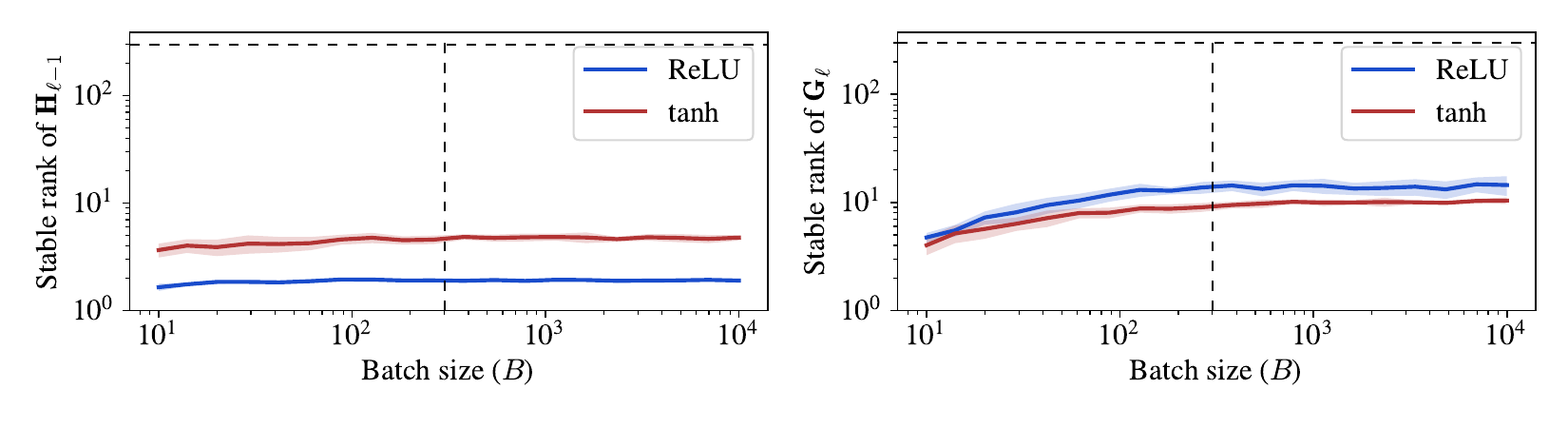}
  \caption{
  Stable ranks for the set of forward hidden vectors $\mH'_{\ell-1}$ and the set of backward hidden vectors $\mG_{\ell}$ for $\ell = 2$.
  Even at large batch size $B$, both are consistently much lower than their maximal possible values of $d = 300$.
  Dashed lines on both axes in both subplots show the network width $d$.
  Note that the stable rank of a matrix $\mM$ is given by $\smash{\norm{\mM}_F^2 / \norm{\mM}_*^2}$.
  }
  \label{fig:forwar\width_backwar\width_sranks}
\end{figure}

\textbf{Experimental details for \cref{fig:mup_vs_ntp_exp}.}
We train MLPs with depth $L = 3$, widths $\width_0 = 3072, \width_1 = \width_2 = d, \width_3 = 1$, and ReLU activation functions.
The data consists of 200 samples from CIFAR-10 \citep{krizhevsky:2009} from only the classes \texttt{airplane} and \texttt{automobile} and uses $\pm 1$ targets.

We use two different hyperparameter schemes as follows.
To implement \muP{}, we take
\begin{equation}
    \sigma_1 = \sqrt{\frac{2}{\width_0}}; \quad
    \sigma_2 = \sqrt{\frac{2}{d}}; \quad
    \sigma_3 = \frac{\sqrt 2}{d}; \quad
    \eta_1 = \eta \cdot \frac{d}{\width_0}; \quad
    \eta_2 = \eta; \quad
    \eta_3 = \eta \cdot \frac{1}{d},
    \tag{\muP{}}
\end{equation}
with global learning rate $\eta = 0.1$.
To implement NTP, we follow \citet{jacot:2018} and \citet{lee:2019-ntk}:
\begin{equation}
    \sigma_\ell = \sqrt{\frac{2}{\width_{\ell - 1}}}; \qquad
    \eta_\ell = \eta \cdot \frac{1}{\width_{\ell - 1}}
    \tag{NTP}
\end{equation}
at all layers, again with $\eta = 0.1$.
These parameterizations are equivalent at $d = 1$, which lets us view each parameterization as a particular scaling prescription applied to a narrow base network. 

We train full-batch for $10^4$ steps, which is sufficient for all widths to drop below $0.01$ training loss on average by the end of training.
We do not expect that training for many more steps would saliently change the resulting plots.
Shaded envelopes in \cref{fig:mup_vs_ntp_exp} denote one standard deviation with respect to both random network initialization and random batch selection over $10$ experiment trials.


It is perhaps worth emphasizing that these experiments worked much better than they had to.
Our theory strictly applies only to the case of a small number of gradient steps relative to network width, but the net updates shown in each subplot of \cref{fig:mup_vs_ntp_exp} reflect the accumulation of thousands of gradient steps, a number which is larger than network width in all cases.
We were thus surprised by the very clear agreement of this experiment with predicted power laws.


\newcommand{\calP}{\mathcal{P}}
\newcommand{\calT}{\mathcal{T}}
\newcommand{\calZ}{\mathcal{Z}}
\newcommand\dout{e}%
\newcommand\trsp{\top}%
\newcommand{\netsort}{\texorpdfstring{{$\textsc{Netsor}\trsp$}}{NetsorT}}
\newcommand{\netsor}{\texorpdfstring{{$\textsc{Netsor}$}}{Netsor}~}
\newcommand{\netsortp}{\texorpdfstring{{$\textsc{Netsor}\top^+$}}{NetsorT+}}
\newcommand{\nexor}{\texorpdfstring{{$\textsc{Ne}{\otimes}\textsc{or}\trsp$}}{NexorT}}
\newcommand{\nexort}{\nexor{}}
\newcommand{\QQ}{\boldsymbol{Q}}
\newcommand{\aout}{a_{\mathrm{out}}}
\newcommand\Gaus{\mathcal{N}}%

\newcommand\EV{\operatorname*{\mathbb{E}}}%

\newcommand\bra[1]{\langle#1 \ob}%

\newcommand\ket[1]{\ob#1 \rangle}%

\newcommand\braket[2]{\langle#1 \ob#2 \rangle}%

\newcommand\hatbrakethat[2]{\hat{\langle}#1 \ob#2 \hat{\rangle}}%

\newcommand\hatbra[1]{\hat{\langle}#1 \ob}%

\newcommand\hatket[1]{\ob#1 \hat{\rangle\mkern-3mu}\mkern+3mu}%

\newcommand\dotbra[1]{\dot{\langle}#1 \ob}%

\newcommand\dotket[1]{\ob#1 \dot{\rangle\mkern-3mu}\mkern3mu}%

\newcommand{\oplim}[1]{\ob#1 \ob}%

\newcommand{\brakethat}[2]{\langle #1 \hatket{#2}}

\newcommand{\braketdot}[2]{\langle #1 \dotket{#2}}

\newcommand{\hatbraket}[2]{\hatbra{#1} #2 \rangle}

\newcommand{\dotbraket}[2]{\dotbra{#1} #2 \rangle}

\newcommand{\Qketdbra}[3]{\overline{\ketdbra{#1}{#2}{#3}}}

\newcommand\opket[2]{\ob #1 \ob #2 \rangle}%

\newcommand\ob{\hstretch{0.7}{\talloblong}}%

\newcommand{\mathbox}[1]{\mathord{\ThisStyle{%
  \fboxsep1\LMpt\relax\kern1\LMpt\fbox{$\SavedStyle#1$}\kern1\LMpt}}}

\newcommand{\bx}{\mathbox}

\newcommand\yy{\boldsymbol{y}}%

\newcommand\zz{\boldsymbol{z}}%
\newcommand\xx{\boldsymbol{x}}%
\newcommand\cc{\boldsymbol{c}}%

\newcommand{\tDW}{\widetilde \DW}

\section{Tensor programs theory}
\label{sec:formal_theory}

For simplicity, assume all hidden widths $\width_1=\cdots=\width_{L-1}$ are the same.
Borrowing from \citet{yang:2021-tensor-programs-IV}, an \emph{abc-parametrization}  is just a recipe for scaling (as powers of width) the multiplier, initializer scale, and learning rate of all parameter tensors of a neural network.
SP, NTP, and $\mu$P are all examples of abc-parametrizations.
A \emph{stable} abc-parametrization is one whose (pre)activations and output do not blow up with width at any step of training. Then (under the generous \cref{assm:TP}) the following theorem is our main result:
\begin{theorem}\label{thm:recovermuP}
In $\mu$P, for almost every learning rate (in the measure-theoretic sense), \cref{{cond:scaling}} is satisfied at any time during training for sufficiently large width.
$\mu$P is the unique stable abc-parametrization with this property.
\end{theorem}

The main insight leading to \cref{thm:recovermuP} is that: $\DW_\ell / \norm{\DW_\ell}_F$ converges to a Hilbert-Schmidt operator (in an appropriate space) as width goes to infinity.
This in turn implies:
%
%
\begin{lemma}\label{cor:DWstablerank}
Unless $\DW_\ell = 0$, $\norm{\DW_\ell}_F / \norm{\DW_\ell}_* = \Theta(1)$ as width goes to infinity.%
\footnote{
It's possible to make precise quantitative statements here using a quantitative version of the Master Theorem but we won't be concerned with it here.}
\end{lemma}

These statements are universal: they hold for any architecture and any adaptive optimizers representable by tensor programs (including convolutional neural networks, residual networks, and transformers, etc., as well as RMSProp, Adam, etc.), not just MLP and SGD.

\subsection{Proofs}
\label{app:TPproofs}

As we have explained the core intuitions of our spectral perspective of feature learning in the main text, here we focus on proving the most general result in the most concise way.

Our proofs will rely on the following notions defined in prior work:
\begin{itemize}
    \item representable architecture \citep[Defn 2.9.1]{yang2023adaptive}
    \item matrix/vector/scalar parameters \citep[Defn 2.9.1]{yang2023adaptive}
    \item abcd-parametrization for representable architectures \citep[Defn 2.9.7]{yang2023adaptive}
    \item entrywise optimizer \citep[Sec 2.1]{yang2023adaptive}
    \item ket and iid-copy notation \citep[Sec 1.2]{yang2023adaptive}
\end{itemize}

Everything here follows under the following:
\begin{assumption}\label{assm:TP}
Assume our architecture is representable by a \nexort{} program with pseudo-Lipschitz nonlinearities, trained by an entrywise optimizer with pseudo-Lipschitz update functions.
\end{assumption}

\subsubsection*{Random Initialization}

The following is our main proposition for initialization:

\begin{proposition}\label{prop:randominit}
    In any abcd-parametrization, any matrix parameter $\mW$ at random initialization has $\norm{\mW}_F / \norm{\mW}_* = \Theta(\sqrt \width)$, but any vector or scalar parameter $\mW$ has $\norm{\mW}_F / \norm{\mW}_* = 1$.
\end{proposition}
\begin{proof}
    This is a claim about random matrices and follows from classical random matrix theory.
    Here's a quick sketch:
    If $\sigma$ is the standard deviation of a matrix entry, then $\norm{\mW}_F \approx \sigma \cdot \width$ from the central limit theorem, and $\norm{\mW}_* \approx 2 \sigma \cdot \sqrt{\width}$ as stated in the main text, from which the first part of the proposition follows.
    For vectors and scalars, the stated ratio is always $1$.
\end{proof}


\subsubsection*{Matrix Updates}

In a tensor program, each vector $x$ of the program converges to a random variable $\ket x$ (called a \emph{ket}) as width goes to infinity in the sense that the scaled inner product $\langle x, y \rangle / \width$ converges almost surely to $\braket x y = \EV \ket x \ket y$.
The kets form a Hilbert space $\calZ$.
Then any weight parameter converges to a linear operator from $\calZ$ to $\calZ$; any vector parameter converges to a linear operator from $\calZ$ to $\R$ or $\R$ to $\calZ$.
(Any scalar parameter converges to a random real, but that is not too important here).

\begin{proposition}\label{prop:HilbertSchmidt}
    Consider any abcd-parametrization.
    For any matrix or vector parameter $\mW$,
    at any step of training, $\DW / \norm{\DW}_F$ converges to a Hilbert-Schmidt integral operator.
\end{proposition}
\begin{proof}
    This is trivial for vector parameters.
    For matrix parameters, observe $\tDW = \DW/ \norm{\DW}_F$ is always a nonlinear outer product
    \[\tDW = Q(\xx; \yy; \cc)\]
    for multi-vectors $\xx$ and $\yy$ and multi-scalars $\cc$ and some nonlinearity $Q$.
    Then the limit $\oplim{\widetilde\DW}$ acts on a ket $\ket z$ by
    \[\oplim{\widetilde\DW} z \rangle = \EV_{\bx 1} Q(\ket \xx; \ket \yy^{\bx 1}; \mathring \cc) \ket z^{\bx 1}\]
    integrating over $\ket \yy$ and $\ket z$.
    The Hilbert-Schmidt norm of $\oplim{\widetilde\DW}$ is
    \[\norm{\oplim{\widetilde\DW}}_{HS} = \sqrt{\EV_{\bx 0, \bx 1} Q(\ket \xx^{\bx 0}; \ket \yy^{\bx 1}; \mathring \cc)^2}\]
    which is finite by the Master Theorem.
    Therefore $\oplim{\widetilde\DW}$ is Hilbert-Schmidt.
\end{proof}


\begin{proposition}\label{prop:paramnorms}
    In \cref{prop:HilbertSchmidt}, 
    for large enough width, unless $\DW = 0$, both spectral norm  and Frobenius norm of $\DW/\norm{\DW}_F$ are $\Theta(1)$.
\end{proposition}
This proposition implies \cref{cor:DWstablerank}.
\begin{proof}
    This is obvious for Frobenius norm since it converges to the Hilbert-Schmidt norm of the operator limit.

    However, the spectral norm cannot be expressed directly in such form.
    But, by definition of spectral norm and of the the ket space $\calZ$, one can construct a nonzero vector $z$ in an extension of the program that defined $\DW$, such that
    \[\norm{\oplim{\tDW} z \rangle} = \theta  \norm{\oplim{\tDW}}_* \norm{\ket z},\]
    for some $\theta \in (1/2, 1]$.%
    \footnote{In fact, we can pick $\theta \in [1-\epsilon, 1]$ for any $\epsilon > 0$.}
    This implies, by the Master Theorem, that
    \[
        \frac{\theta}{2}\frac{\norm{\tDW z}}{\norm{z}} \le \norm{\tDW}_* \]
    for sufficiently large width.
    The LHS is $\Theta(1)$, so we are done.

\end{proof}

\begin{proposition}\label{prop:muPmatrix_param}
    In $\mu$P,
    at any fixed time during training, as width $\to \infty$, any matrix parameter $\mW$ has spectral norm $\Theta(1)$ and Frobenius norm $\Theta(\sqrt \width)$.
\end{proposition}
\begin{proof}
    In $\mu$P, it's trivial to see that $\norm{\DW}_F=\Theta(1)$.
    So any nonzero $\DW$ (at any point of training) has both spectral norm  and Frobenius norm $\Theta(1)$ by the above.
    
    Simple calculation then shows that $\mW$ at any fixed time has $\Theta(\sqrt \width)$ since this is the case at initialization by \cref{prop:randominit}.
    This means the quadratic mean of the singular values of $\mW$ is $\Theta(1)$, so its max singular value must be $\Omega(1)$.
    But it furthermore must be $\Theta(1)$ because $\mW$ at initialization and all of its updates have $O(1)$ spectral norm.
\end{proof}

\begin{theorem}\label{thm:main_abcd}
In $\mu$P, for all but a measure-zero set of learning rates, \cref{{cond:scaling}} is satisfied at any time during training for sufficiently large width.
$\mu$P is the unique stable and faithful abcd-parametrization with this property.
\end{theorem}
\begin{proof}
    In $\mu$P,
    by \cref{prop:muPmatrix_param}, all matrix parameters satisfy \cref{{cond:scaling}} no matter what the learning rate is.
    However, for vector parameters, it's possible for some specific learning rate to cause the weights to vanish after an update, but at most a (Lebesgue) measure-zero set of learning rates will cause this to happen.
    Assuming this vanishing does not happen, $\mW$ has $\Theta(1)$ Frobenius norm at initialization and at all times during training.
    By \cref{prop:paramnorms}, $\mW$ also has $\Theta(1)$ spectral norm, so \cref{{cond:scaling}} is satisfied.
    A similar but easier argument applies to all scalar parameters.
    This shows $\mu$P satisfies \cref{{cond:scaling}}.

    Since any other stable and faithful parametrization essentially just rescales the initialization and the update, we see that no other parametrization can satisfy \cref{{cond:scaling}}.\end{proof}

For SGD, all abc-parametrizations are equivalent to a faithful abcd-parametrization because $Q$ is identity, so \cref{thm:main_abcd} recovers \cref{thm:recovermuP}.

\section{Empirical checks of \cref{assumption_1,assumption_2,assumption_3}}
\label{app:checking_assumptions}

In \cref{sec:warmup_and_extensions}, we first illustrated how \cref{cond:scaling} satisfies \cref{des:scaling} in a minimal model---a deep linear network trained for one step on one sample---and then iteratively extended our argument to multiple steps, nonlinearities, and multiple samples.
Each extension came with a mild assumption.
These assumptions are intended to be natural conditions one expects from generic network dynamics.

In this section, we restate each assumption, explain further why it is expected to hold, and then present a validating experiment.
All experiments use the same setup as that of \cref{fig:mup_vs_ntp_exp}.
Plotted quantities depending on the sample $\vx$ are averaged over all $\vx$ in the batch, with shaded envelopes showing one standard deviation of this mean over five experiment trials.

\theoremstyle{plain}
\newtheorem{assumptionrepeat}{Assumption}

\begin{assumptionrepeat}
Updates do not perfectly cancel initial quantities.
That is:
\begin{align}
    \norm{\mW_\ell + \DW_\ell}_*
    &=
    \Theta \left( \norm{\mW_\ell}_* + \norm{\DW_\ell}_* \right) \\
    \norm{\vh_\ell(\vx) + \Delta \vh_\ell(\vx)}_2
    &=
    \Theta(\norm{\vh_\ell(\vx)}_2 + \norm{\Delta \vh_\ell(\vx)}_2).
\end{align}
\end{assumptionrepeat}

Recall that the cancellation of two high-dimensional tensors (i.e. matrices or vectors) tends to be unlikely as dimension grows (unless there is a good reason to expect cancellation, which there is not here).
For the small learning rate used in our experiment, \cref{assumption_1} is in fact true simply because small updates are not big enough to cancel the initial quantities even if aligned.
In \cref{fig:assumption_1_check}, we verify that \cref{assumption_1} also holds when $\DW_\ell$ and $\Delta \vh_\ell(\vx)$ are respectively replaced by $\mW_\ell - \mW_\ell^0$ and $\vh_\ell(\vx) - \vh^0_\ell(\vx)$, the total updates across all of training.

\begin{figure}[H]
  \centering
  \includegraphics[width=15cm]{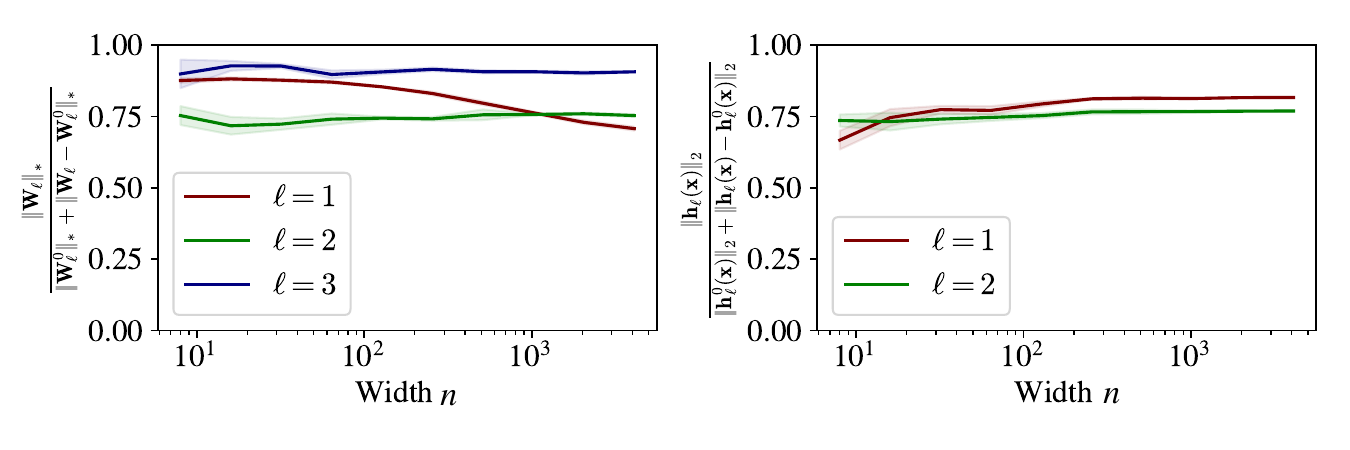}%
    \llap{\raisebox{4.8mm}{
    \includegraphics[width=2.3mm, clip=true, trim = 49mm 8mm 24mm 37mm]{img/n_example.pdf}%
    }\hspace*{25.4mm}}%
    \llap{\raisebox{4.8mm}{
    \includegraphics[width=2.3mm, clip=true, trim = 49mm 8mm 24mm 37mm]{img/n_example.pdf}%
    }\hspace*{98.5mm}}%
  \caption{
    \textbf{Verification of \cref{assumption_1}.}
    Subplots show the ratios between left- and right-hand sides of conditions of \cref{assumption_1} at various layers.
    Envelopes show variation over trials.
    The ratios are $\Theta(1)$ (showing no obvious decay with network width), verifying Assumption 1.
  }
  \label{fig:assumption_1_check}
\end{figure}

\begin{assumptionrepeat}
    $\norm{\vh'_\ell(\vx)}_2 = \Theta\left(\norm{\vh_\ell(\vx)}_2\right).$
\end{assumptionrepeat}

Recall that $\vh'_\ell(\vx) = \phi(\vh_\ell(\vx))$, and that $\vh_\ell(\vx)$ contains elements of size $\Theta(1)$.
Satisfaction of \cref{assumption_2} merely requires that $\phi$ maps a nonvanishing fraction of preactivations to nonzero quantities.
Its violation would require preactivations to concentrate in regions of $\mathbb{R}$ which $\phi$ maps to zero.
This might occur, for example, in the unlikely scenario in which $\phi = \textrm{ReLU}$ and almost all preactivations are negative.\footnote{It will also occur if the activation function is $\phi(z) = 0$, but we can neglect this edge case.}
For activations that do not map nonzero inputs to zero (for example, $\phi = \tanh$), \cref{assumption_2} may be dropped altogether.
\cref{fig:assumption_2_check} verifies \cref{assumption_2} in a deep ReLU MLP.
See \citet{poole:2016} for a theoretical framework sufficient to check \cref{assumption_2} at initialization.

\begin{figure}[H]
  \centering
  \includegraphics[width=8cm]{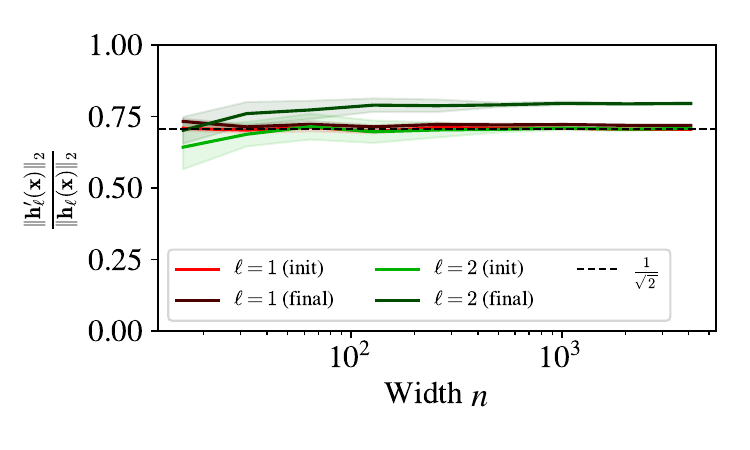}%
    \llap{\raisebox{4.6mm}{
    \includegraphics[width=2.2mm, clip=true, trim = 49mm 8mm 24mm 37mm]{img/n_example.pdf}%
    }\hspace*{27.7mm}}%
  \caption{
    \textbf{Verification of \cref{assumption_2}.}
    At all layers, the norm of the activation vector scales as the norm of the corresponding preactivation vector, both before and after training.
    (Note that the ratio in this case is close to $1/\sqrt{2}$, which is what one expects from an approximately-mean-zero random variable passed through a ReLU nonlinearity.)
  }
  \label{fig:assumption_2_check}
\end{figure}

\begin{assumptionrepeat}
    $\norm{\DW_\ell \vh_\ell(\vx_i)}_2 =
    \Theta(
    |\!|
    \frac{1}{B}
    \DW_\ell^{(i)} \vh_\ell(\vx_i) |\!|_2
    )$.
\end{assumptionrepeat}

Like \cref{assumption_1}, a violation of \cref{assumption_3} would require a perfect cancellation of high-dimensional matrices, which is unlikely.
We verify \cref{assumption_3} in \cref{fig:assumption_3_check}.
\cref{assumption_3} is also verified implicitly by the right-hand subplot of \cref{fig:batch_scaling}.

\begin{figure}[H]
  \centering
  \includegraphics[width=8cm]{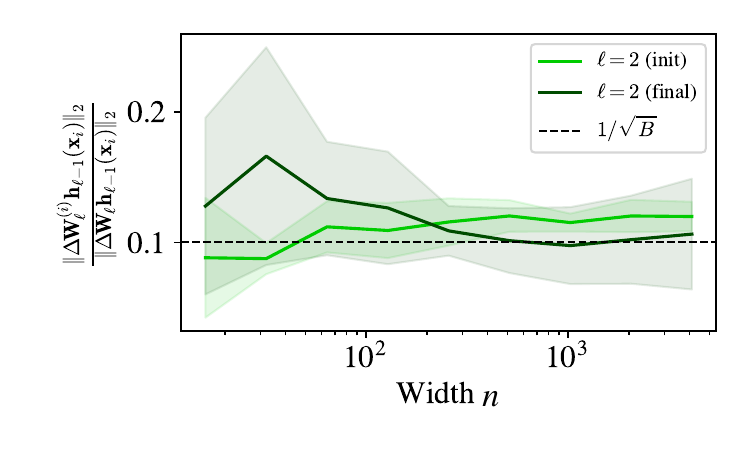}%
    \llap{\raisebox{4.55mm}{
    \includegraphics[width=2.2mm, clip=true, trim = 49mm 8mm 24mm 37mm]{img/n_example.pdf}%
    }\hspace*{26.5mm}}%
  \caption{
    \textbf{Verification of \cref{assumption_3}.}
    We compute $
    \|\DW_\ell^{(i)} \vh_\ell(\vx_i) |_2 / \norm{\DW_\ell \vh_\ell(\vx_i)}_2$ at layer $\ell = 2$ for a single step both at initialization and after training.
    As expected, this ratio remains $\Theta(1)$ as width grows --- note the lack of decay with width.
    We note as a curiosity that this quantity hovers around $1 / \sqrt{B}$ (dashed line), which one expects from the central limit theorem if, for example, $\DW_\ell$ resembles a sum of $B$ terms similar to $\DW_\ell^{(i)}$ but with random signs.
  }
  \label{fig:assumption_3_check}
\end{figure}

\section{Scalings for biases}
\label{app:biases}

Here we extend our spectral analysis to biases.
Let $\vb_\ell \in \R^{\width_\ell}$ be a bias vector which enters during forward propagation as $\vh_\ell(\vx) = \mW_\ell \vh'_{\ell-1}(\vx) + \vb_\ell$.
We may choose to view the bias vector as a weight matrix $\vb_\ell \in \R^{\width_\ell \times 1}$ connecting an auxiliary layer with width $1$ and output $1$ to the $\ell$th hidden layer, after which we may simply apply our scaling analysis for weight matrices.
The spectral scaling condition (\cref{cond:scaling}) prescribes that $\norm{\vb_\ell}_2 = \Theta(\sqrt{\width_\ell})$ and $\norm{\Delta \vb_\ell}_2 = \Theta(\sqrt{\width_\ell})$, and \cref{pzn:spectral_pzn} prescribes that the initialization scale and learning rate should be $\sigma^{b}_\ell = \Theta(1)$ and $\eta^{b}_\ell = \Theta(\width_\ell)$.
In practice, one may usually just take $\sigma^{b}_\ell = 0$.

\section{Nondimensionalization and natural norms}
\label{app:nondim}

While there are conventional notions of vector norms and matrix spectral norms, one can equip any norm on the input and output spaces of a matrix, thereby inducing a corresponding spectral norm on it. 
For example, the conventional spectral norm of a matrix is defined assuming the typical $\ell^2$ norm on the input and output spaces; if the input space is instead given the $\ell^1$ norm, then the spectral norm is something very different. 
As we explain below, the vectors encountered in deep neural networks naturally come with a specific notion of norm, which we refer to as the ``natural norm.'' When we equip vectors at various layers of the neural network with their natural norms, we induce what we call a ``natural spectral norm'' on the weight matrices. Our primary criterion \cref{cond:scaling} is then simplified to stipulating that these \emph{natural spectral norms} should be $\Theta(1)$ in the limit of large network width (\cref{cond:scaling_nondim}).

\subsection{Dense and sparse vectors}
Vectors in the context of deep learning typically fall into one of two categories: dense and sparse.
Dense vectors are characterized by having every entry contribute a constant amount to the squared Euclidean norm of the vector, while sparse vectors have only a constant number of entries that do so.
Here, the term ``constant'' means $\Theta(1)$ with respect to the size of the vector.%
\footnote{Actually, a vector of $d$ iid samples from a unbounded subgaussian or subexponential distribution will have its extreme elements scale like $\mathrm{polylog}(d)$, so the really correct thing to say here is $\tilde \Theta(1)$, but in the main text we convey the key intuition in the main text without being too pedantic.
This fine detail does not affect the conclusion of this section.}

For example: in the input layer, image data typically takes the form of dense vectors, while the one-hot encoding in language models takes the form of sparse vectors. The output vector of a network is typically a dense vector. All hidden vectors (pre- or post-activation) are dense.

\subsection{Defining natural norms}

\begin{definition}[Natural $\ell^2$-norm]
The \emph{natural $\ell^2$-norm} for a \emph{dense} vector $\vv \in \R^m$ is the RMS norm
\begin{equation}
    \nnorm{\vv}_{\tilde{2}} \defeq \frac{1}{\sqrt{m}} \norm{\vv}_2.
\end{equation}
The \emph{natural $\ell^2$-norm} for a \emph{sparse} vector $\vv$ is simply the usual $\ell^2$-norm:
\begin{equation}
    \nnorm{\vv}_{\tilde{2}} := \norm{\vv}_2.
\end{equation}
\end{definition}

While we say ``dense'' or ``sparse'' vector, these adjectives really apply to the space that contains such a vector.
For example, the set of one-hot encodings form a ``sparse'' input space, while the pre-activations at layer $\ell$ forms a ``dense'' hidden space.

\begin{definition}[Natural spectral norm]
Given a parameter matrix $\mA$, equip its input and output spaces with their natural norms. Then the \emph{natural spectral norm} $\nnorm{\mA}_{\tilde{*}}$ of $\mA$ is defined as the induced spectral norm.
\end{definition}

For example, if both the input and output spaces are dense, then the natural spectral norm of $\mA \in \R^{m \times n}$ is
\begin{equation}
    \nnorm{\mA}_{\tilde{*}} \defeq \frac{\sqrt{m}}{\sqrt{n}} \norm{\mA}_*.
\end{equation}

\subsection{Spectral scaling in natural norms}
By equipping vectors and matrices with their natural norms, we simplify the conditions \cref{des:scaling,cond:scaling} from the main text:

\begin{tcolorbox}[colback=white, colframe=black, boxrule=1pt, arc=0pt]
  \begin{desired}[Feature learning, natural norms] Let $\vh_\ell(\vx) \in\R^{\width_\ell}$ denote the features of input $\vx$ at layer $\ell$ of a deep neural network, and let $\Delta \vh_\ell(\vx)\in\R^{\width_\ell}$ denote their change after a gradient update. We desire that:
    \begin{align*}
        \nnorm{\vh_\ell}_{\tilde{2}} = \Theta (1)
        \quad \text{and} \quad
        \nnorm{\Delta \vh_\ell}_{\tilde{2}} = \Theta (1),
     \quad \text{ at layers } \ell=1,...,L\!-\!1.
    \end{align*}
\end{desired}
\end{tcolorbox}

Note the norm here is just RMS norm because we only talk about hidden vectors (which are always dense).

\begin{tcolorbox}[colback=white, colframe=black, boxrule=1pt, arc=0pt]
  \begin{condition}[Spectral scaling, natural norms]
  \label{cond:scaling_nondim}
    Consider applying a gradient update $\DW_\ell \in \R^{\width_\ell\times \width_{\ell-1}}$ to the $\ell$th weight matrix $\mW_\ell\in\R^{\width_\ell \times \width_{\ell-1}}$.
    The spectral norms of these matrices should satisfy:
    \begin{align*}
        \nnorm{\mW_\ell}_{\tilde{*}} = \Theta \left( 1 \right)
        \ \ \ \text{and} \ \ \
        \nnorm{\DW_\ell}_{\tilde{*}} = \Theta \left( 1 \right),
     \quad \text{ at layers } \ell=1,...,L.
    \end{align*}
\end{condition}
\end{tcolorbox}




In summary, using these rescaled norms (that we call ``natural norms''), our problem is nondimensionalized: feature vectors, feature vector updates, weight matrices, and weight matrix updates are $\Theta(1)$ in norm.
These natural norms provide a universal framework that covers specific cases, such as one-hot embeddings in language models, that are not handled directly by \cref{cond:scaling}.





\end{document}